\documentclass[journal,comsoc,twoside]{IEEEtran}
\usepackage[T1]{fontenc}

\usepackage{graphicx}
\usepackage{xcolor}
\usepackage{hyperref}
\usepackage{amsmath}
\usepackage{mathrsfs}
\usepackage{mathtools}
\usepackage{amsfonts}
\usepackage{latexsym}
\usepackage{amssymb}
\usepackage{wasysym}
\usepackage{calrsfs}
\usepackage{tabto}
\usepackage[utf8]{inputenc}

\usepackage[ruled,vlined]{algorithm2e}
\graphicspath{ {./images/} }
\usepackage{amsthm}

\ifCLASSINFOpdf
\else
\fi
\usepackage{amsmath}
\usepackage[cmintegrals]{newtxmath}
\newcommand{\subjectto}{{\text {subject to}}}
\DeclareMathOperator*{\argmin}{\arg\!\min}

\begin{document}
%
\title{Distributed Sparse Feature Selection in Communication-Restricted Networks}
%
%
%

\author{Hanie~Barghi$^*$,~
        Amir~Najafi$^\dagger$,~
        and~
        Seyed~Abolfazl~Motahari$^*$
\thanks{$*$~Data Analytics Lab (DAL), Department
of Computer Engineering, Sharif University of Technology, Tehran,
Iran.~E-mails: hbarghi@ce.sharif.edu,~motahari@sharif.edu}
\thanks{$\dagger$~School of Mathematics, Institute for Research in Fundamental Sciences (IPM), P.O. Box: 19395-5746, Tehran, Iran.~E-mail: najafi@ipm.ir}
}

%
%

\markboth{IEEE Transactions on Signal Processing (submitted)}{Barghi \MakeLowercase{\textit{et al.}}: Distributed Sparse Feature Selection in Communication-Restricted Networks}
%



\maketitle

\renewcommand{\qedsymbol}{$\blacksquare$}
\newcommand{\norm}[1]{\left\lVert#1\right\rVert}
\newtheorem{theo}{Theorem}
\newtheorem{lemm}{Lemma}
\newtheorem{corl}{Corollary}
\DeclarePairedDelimiter\abs{\lvert}{\rvert}%

\begin{abstract}
This paper aims to propose and theoretically analyze a new distributed scheme for sparse linear regression and feature selection. The primary goal is to learn the few causal features of a high-dimensional dataset based on noisy observations from an unknown sparse linear model. However, the presumed training set which includes $n$ data samples in $\mathbb{R}^p$ is already distributed over a large network with $N$ clients connected through extremely low-bandwidth links. Also, we consider the asymptotic configuration of $1\ll N\ll n\ll p$. In order to infer the causal dimensions from the whole dataset, we propose a simple, yet effective method for information sharing in the network. In this regard, we theoretically show that the true causal features can be reliably recovered with negligible bandwidth usage of $O\left(N\log p\right)$ across the network. This yields a significantly lower communication cost in comparison with the trivial case of transmitting all the samples to a single node (centralized scenario), which requires $O\left(np\right)$ transmissions. Even more sophisticated schemes such as ADMM still have a communication complexity of $O\left(Np\right)$. Surprisingly, our sample complexity bound is proved to be the same (up to a constant factor) as the optimal centralized approach for a fixed performance measure in each node, while that of a na\"{i}ve decentralized technique grows linearly with $N$.  Theoretical guarantees in this paper are based on the recent analytic framework of debiased LASSO in \cite{javanmard2018}, and are supported by several computer experiments performed on both synthetic and real-world datasets.
\end{abstract}

\begin{IEEEkeywords}
Distributed Learning, Debiased LASSO, Sparse Feature Selection, Communication Cost, Sample Complexity
\end{IEEEkeywords}

%
\IEEEpeerreviewmaketitle

\section{Introduction}

\IEEEPARstart{D}{istributed} machine learning has become increasingly popular nowadays due to the abundance of real-world data as well as the fact that data instances are split among several connected machines. For instance, smartphones have already been transformed into powerful sensors that constantly gather information through their cameras and microphones. These devices are frequently carried by millions of people around the globe and generate huge amounts of distributed data while remain connected to a few servers in a large network, e.g., the Internet. Although such data can greatly improve the quality of AI-based services on the server-side, the computational budget and more importantly the required bandwidth to transmit the samples to a single server are prohibitively expensive. Moreover, the transmission of raw user data raises serious concerns w.r.t. privacy. The limitations mentioned above, followed by the fact that client-side devices of this era are also highly-capable computing machines, have motivated researchers to develop schemes for learning from data in a distributed and privacy-preserving manner.  


Training machine learning models in each client in isolation faces two major obstacles. First, the high dimension of feature vectors does not naturally match the insufficient size of local datasets in each client. This issue increases the risk of overfitting and makes the locally trained models to be unreliable. Second, clients should not (due to privacy) or cannot (due to bandwidth restriction) transmit data samples to each other or to a server node, which constrains the amount of transmitted information not to exceed a certain threshold. This prevents us from accumulating all the data on a single server.

This problem has been the center of attention in a variety of research areas \cite{Wangcmef, Aijuncom, JMLR:v21:20-297,JMLR:v18:16-002, pmlr-v130-bao21a, 10.5555/3122009.3290415}. In this work, we set out to provide a solution for a particular case of interest: sparse linear feature selection in high dimensions when samples are distributed over a network of size $N\ge 1$ with at least one central node. We assume the links in the network to have extremely low bandwidths, such that the amount of transmitted information per link may not exceed $\tilde{O}\left(1\right)$ w.r.t. dimension $p$ and sample size $n$. The core idea behind our solution is to first, identify an initial set of supposedly causal dimensions on each client through a weakly-regularized variant of debiased LASSO estimator \cite{javanmard2018}. All such dimension sets can be passed to a server with negligible communication cost. Then, a simple majority-voting scheme aggregates the received information at the server-side and outputs an appropriate final set for causal dimensions. This final set can then be shared with all the clients, again with a negligible number of transmissions per link. We both theoretically and experimentally show that the accuracy of selected features returned by our method significantly surpasses those that can be achieved in each client, separately (Decentralized scenario). In fact, our achieved accuracy in terms of true/false positive rates is surprisingly close to that of the optimal case in which all data samples are already accumulated in a central server node (centralized scenario). 

From now on, let us denote by $\rho>0$ the signal-to-noise ratio of the unknown sparse model, which is the ratio of the non-zero coefficient values in the linear model to the standard deviation of the observation noise. Mathematically speaking, we prove that for a signal-to-noise ratio of $\rho$, a \emph{proper} asymptotic configuration which we define and discuss in details in Section \ref{performanceanalysis}, and as long as $N$ is appropriately bounded, the minimum-required sample size in order to guarantee a false positive rate of $\epsilon\log p/p$ and true positive rate of $1-\delta$, for $\epsilon,\delta>0$, is $n\ge \Theta\left(\rho^{-2}\log\frac{1}{\epsilon\delta}\right)$. However, this bound corresponds to the optimal centralized scenario where all data patches are already transmitted to a single server node. Hence, this strategy imposes a prohibitively high communication cost of $O\left(np\right)$ total transmissions. In another extreme scenario, the learner may wish to train each node in isolation, i.e., the fully decentralized case. This reduces the communication cost to $O(1)$, however, the minimum required sample size in order to guarantee that \emph{all} nodes are doing fine would inevitably increase to $n\ge \Theta\left(N\rho^{-2}\frac{1}{\epsilon\delta}\right)$, which is not practical in most real-world situations. 

Our proposed method effectively mitigates between sample complexity and communication cost while also distributes the computational complexity all over the network. More precisely, the proposed method offers a communication cost of $O\left(N\log p\right)$, while requires a total sample size of $n\ge\Theta\left(\rho^{-2}\log\frac{1}{\epsilon\delta}\right)$ which is the same (up to a vanishing constant factor) as the optimal centralized case. See Table \ref{table:theoryPerformance} for a more informative comparison. Considering the fact that we usually have $p\gg n\gg N\gg1$ in practice, the above bounds imply a significant reduction in the number of transmissions without the need to increase the sample complexity compared to the fully-decentralized case. Additionally, the proposed method preserves privacy since it does not transmit raw user data. Our theoretical analysis is based on a number of fundamental guarantees for debiased LASSO regression, which have been recently proved in \cite{javanmard2018}. Moreover, the experimental validations in this paper consider both synthetic and real-world data, where we show a comparable performance has been achieved with rival methods which either impose more communication cost, or lack theoretical guarantees for their solutions.

\begin{table}[t]
\caption{Sample complexity, communication cost, and computational complexity tradeoff between our proposed method and two extreme scenarios. Please note that we usually have $1\ll N\ll n \ll p$ in practice, which shows the extent of improvements in this work. 
}
\label{table:theoryPerformance}
\begin{tabular}{l l l l}
& &\\ 
\hline
\\[-2mm]
& 
Minimum required &
Total commu-&
Computation-
\\
& 
sample size, $n\ge$ &
nication cost &
al complexity
\\[1mm]
\hline
\\[-3mm]
Decentralized  & 
$\Theta\left(N\rho^{-2}\log\frac{1}{\epsilon\delta}\right)$ & 
$O(1)$ & 
$O\left(Nnp^2+p^3\right)$
\\
Centralized &
$\Theta\left(\rho^{-2}\log\frac{1}{\epsilon\delta}\right)$ & 
$O(np)$ & 
$O\left(np^2+p^3\right)$
\\
\textbf{Proposed} &
$\Theta\left(\rho^{-2}\log\frac{1}{\epsilon\delta}\right)$ & 
$O(N\log p)$ & 
$O\left(np^2+p^3\right)$
\\
\hline
\end{tabular}
\end{table}

\subsection{Related Works}
The most popular techniques in distributed learning are based on parameter sharing in a star-shaped network \cite{McMahan2017CommunicationEfficientLO, 10.1561/2200000016, 7178675, 9116528}. Here, the server node is supposed to constantly send the aggregated model parameters to all the clients, where each client can then update parameters based on its local data. The server again aggregates the updated parameters sent from each client and iterates this process until some convergence criteria are satisfied. The alternating direction method of multipliers (ADMM) is a well-known technique in this family, first developed in \cite{M2AN_1975__9_2_41_0} and \cite{GABAY197617}, and then used to solve distributed convex optimization problems in \cite{10.1561/2200000016}. Also, the convergence analysis of ADMM under any convex or sub-convex regime is discussed in \cite{pmlr-v37-nishihara15} and \cite{hong2017linear}. ADMM primarily requires a star-shaped network, however, a newer version called Distributed-ADMM has been recently introduced in \cite{mota2013d} which works on any connected graph. Another noticeable method for training distributed deep networks is the FederatedAveraging algorithm in \cite{McMahan2017CommunicationEfficientLO}.

Recently, distributed sparse linear regression has become increasingly popular; For instance, \cite{HUANG2020107497, 6375851} employ diffusion adaptation to maintain the sparsity of the learned linear models. Generally speaking, a central idea to guarantee sparsity in such models is to use $\ell_1$-regularization; For example, see \cite{8821568, locallyconvex, pmlr-v70-wang17f}. The mentioned methods share all the model parameters in the network, making the communication cost be at least the order of data dimension. Notably, \cite{8821568} introduced a debiased LASSO model for a network-based setting that comes with some theoretical convergence guarantees. In \cite{locallyconvex}, communication cost is the order of the data dimension, but there is an upper-bound for the steps needed for convergence. It is in contrast to methods such as ADMM, which suffer from slow convergence. In another related line of work, some existing algorithms are based on iterative hard thresholding (IHT) to keep the model parameter sparse, e.g., \cite{iht1, BLUMENSATH2009265}. Additionally, \cite{6638510, 7178675} develop the IHT grounded methods for distributed sparse learning, where \cite{7178675} in particular highly concentrates on communication cost reduction.

The high-dimensional nature of the majority of real-world problems has made feature selection to be an essential stage in modern machine learning. In particular, the importance of dimension reduction in distributed learning has motivated researchers to develop methods for {\it {distributed feature selection}}; see \cite{HODGE201624, MORANFERNANDEZ201727, LOPEZ2021124, Tadist2021SDPSOSD, GONZALEZDOMINGUEZ2019399}. A number of recent approaches in this line of research concentrate on reducing the communication cost \cite{9116528, Gui2020ADAGESAA}. In particular, the method proposed by \cite{Gui2020ADAGESAA} has been claimed to effectively decrease the communication cost by applying a thresholding method; However, the mentioned works lack explicit theoretical guarantees for their claimed performance improvements.

The rest of the paper is organized as follows: Mathematical definition of the problem is given in Section \eqref{problemdef}. Our proposed method is discussed in Section \eqref{promthd}. In Section \eqref{performanceanalysis}, we theoretically analyze our method and present some formal performance guarantees. Section \ref{sec:exp} is devoted to experiments on both synthetic and real-world data. Finally, conclusions are made in Section \ref{sec:conc}.
\section{Problem Setting and Main Results} \label{problemdef}

We consider a linear sparse Gaussian model
\begin{equation}
    \label{linearmodel}
    Y = X\theta^*+\omega,    
\end{equation}
where $ X\in \mathbb{R}^{n\times p}$ and $Y\in \mathbb{R}^n$ are the design matrix and observation vector, respectively. Also, $\omega \sim \mathcal{N}(0, \sigma^2 I_{n\times n})$ is the additive white observation noise. Parameters $n$ and $p$ represent the sample size and data dimension, respectively. The rows of design matrix $X$ are i.i.d instances of a zero-mean multivariate Gaussian distribution with covariance matrix $\Sigma\in\mathbb{R}^{p\times p}$. We have assumed $\Sigma$ to be positive definite with a bounded condition number\footnote{It should be noted imposing no restrictions on the rank of $\Sigma$ could turn feature selection into an impossible task. For example, assume $\Sigma$ is rank-$1$, then all dimensions in each row of $X$ become identical up to constant factors, and the model becomes unlearnable.}. Also, without loss of generality, let us assume all the diagonal entries of $\Sigma$ are $1$. We have examined both cases of 1) known and 2) unknown covariance matrix in this work. For the latter case, i.e., unknown $\Sigma$, some further assumptions on $\Sigma$ are required for the analysis to be non-trivial \cite{javanmard2019}. For example, we assume that the dimensions of the presumed multivariate Gaussian can only be weakly dependent, which means the precision matrix (inverse of covariance) should be sparse. More relaxed arguments can replace the above assumption with the help of more sophisticated mathematical tools. However, this approach goes beyond the scope of this paper.

The latent parameter $\theta^* \in \mathbb{R}^{p}$ is assumed to be $s_0$-sparse, i.e. it has $s_0$ nonzero entries with $s_0\ll p$, and the support of $\theta^*$, denoted by $T$, can be any subset of size $s_0$ from the dimension set $[p]$. In order to simplify some of our asymptotic results, let us assume $s_0\leq O\left(\log p\right)$. Moreover, nonzero entries of $\theta^*$ are assumed to satisfy the following condition:
\begin{equation*}
\abs{\theta^*_i} \ge \beta\quad\forall{i \in T}.
\end{equation*}
This assumption is necessary for feature selection, since if any dimension in the support of $\theta^*$ becomes infinitesimally small, then perfect recovery of $T$ becomes asymptotically impossible. Throughout the paper, we refer to the quantity $\beta/\sigma$ as Signal-to-Noise Ratio (SNR). Finally, we assume data is row-wise distributed amongst $N\ge 1$ clients. For simplicity, we assume each client has access to $\left\lfloor n/N\right\rfloor$ data samples. However, generalizing the results to non-equal data subsets is straightforward as long as each data patch still has a size of $\Theta\left(n/N\right)$ in the asymptotic regime.

The problem is to find an estimator for $T$, denoted by $\hat{T}\subset[p]$, such that 1) the number of False Positives (FP) defined as
$$
\mathbf{FP}=\left\vert \hat{T}-T\right\vert
$$
remains $O\left(\log p\right)$ with an arbitrarily small constant, and 2) the number of True Positives (TP) defined as
$$
\mathbf{TP}=\left\vert T\cap\hat{T}\right\vert
$$
becomes arbitrarily close to $s_0$, all with high probability. Here, $\abs{\cdot}$ denotes the cardinality of a set. The following theorem states our main results:
\begin{theo}[Main Results]
\label{thm:main}
Assume the sparse linear model described in \eqref{linearmodel} such that the unknown coefficient vector $\theta^*$ is $s_0$-sparse with $s_0\leq O\left(\log p\right)$. Also, let total sample size to satisfy $n\ge \Theta\left(s_0\log p\right)$. Suppose data samples are distributed amongst $N$ clients with 
$$
N\leq \min\left[O\left(\frac{n}{s_0\log p}\right),
o\left(p\right)\right],
$$
while all clients are connected to at least one central server node ({\textsc{Bounded Network Size}}). Finally, assume SNR satisfies the following condition ({\textsc{Non-negligible SNR}}):
$$
\frac{\beta}{\sigma}
\ge \Theta\left[
\frac{\sqrt{s_0}\log p}{n/N}+
\sqrt{\frac{\log p}{n/N}}
\right].
$$
Then, the necessary and sufficient conditions on sample complexity and communication cost in order to guarantee
$$
\mathbf{FP}\leq \epsilon\log p
\quad,\quad
\mathbf{TP}\ge s_0\left(1-\delta\right)
$$
for $\epsilon,\delta>0$ are as follows:
\begin{itemize}
\item
(\textbf{centralized}): having $n\ge \Theta\left(\left[\sigma/\beta\right]^{2}\log\frac{1}{\epsilon\delta}\right)$ and a minimum of $O\left(np\right)$ transmissions across the network, guarantees the above goals with probability at least $1-O\left(p^{-1}\right)$.
\item
(\textbf{decentralized}): having $n\ge \Theta\left(N\left[\sigma/\beta\right]^{2}\log\frac{1}{\epsilon\delta}\right)$ and at most $O\left(1\right)$ overall transmissions, guarantees the above goals with probability at least $1-O\left(Np^{-1}\right)=1-o(1)$.
\end{itemize}
Finally, there exists a polynomial-time algorithm such that the same guarantees on true and false positives hold with probability at least $e^{-O\left(Np^{-1}+N^{-1}\right)}=e^{-o(1)}$, as long as the following conditions are satisfied:
\begin{itemize}
\item 
(\textbf{The proposed}): having $n\ge\Theta\left(\left[\sigma/\beta\right]^{2}\log\frac{1}{\epsilon\delta}\right)$, and making $O\left(N\log p\right)$ transmissions across the network.
\end{itemize}
\end{theo}

It should be noted the above assumptions, i.e., bounded network size and non-negligible SNR conditions, are fundamental and cannot be bypassed. In fact, increasing the network size $N$ while $n$ and $p$ are fixed simply means the number of data samples in each client is being reduced toward zero. And at some point, decisions in each client become so unreliable that makes the correct recovery almost impossible, unless actual samples are transmitted. A similar argument holds for the SNR $\beta/\sigma$.

The proof of Theorem \ref{thm:main} is given according to the following agenda: First, we propose a simple and efficient algorithm in Section \ref{promthd} which is claimed to achieve the final (proposed) sample complexity and communication cost bounds in Theorem \ref{thm:main}. Then, we theoretically analyze its performance in Section \ref{performanceanalysis} and concurrently analyze the performance of centralized and decentralized cases, as well.

\section{The Proposed Algorithm}
\label{promthd}

Before formally specifying our algorithm, first let us review some key aspects of sparse regression and feature selection via LASSO algorithm \cite{471413, lasso}. Given the pair $\left(X,Y\right)$, model parameters $\theta^*$ can be learned via LASSO regression:
\begin{equation}
    \label{lasso}
    \hat{\theta}^{\mathrm{Lasso}} = \argmin_{\theta\in\mathbb{R}^p}~\left\{
    \ell_{\lambda}\left(\theta;X,Y\right)\triangleq
    \norm{Y-X\theta}_2^2 + \lambda \norm{\theta}_1\right\},
\end{equation}
where $\lambda\ge0$ is a user-defined regularizing coefficient that controls the sparsity level of $\hat{\theta}^{\mathrm{Lasso}}$. Also, $\left\Vert\cdot\right\Vert_1$ and $\left\Vert\cdot\right\Vert_2$ denote the $\ell_1$ and $\ell_2$ norms, respectively. LASSO regression has been extensively analyzed in previous works, cf. \cite{6866880, lialasso, wainwright_2019}. Here, we are particularly interested in feature selection through investigating the support of the LASSO estimator in \eqref{lasso}. To this end, it has been already shown that the ``debiased LASSO" estimator could be a better alternative \cite{javanmard2018,zhang2011confidence,van_de_Geer_2014,6866880}, which is defined as:  
\begin{equation}
\label{dbe}
\hat{\theta}^d = \hat{\theta}^{\mathrm{Lasso}} + \frac{1}{n} M X^T \left(Y-X\hat{\theta}^{\mathrm{Lasso}}\right).
\end{equation}
Assuming covariance matrix $\Sigma$ is unknown, $M$ in \eqref{dbe} represents an estimation of the inverse covariance matrix $\Theta = \Sigma^{-1}$. Calculating $M$ requires a number of precise steps which are thoroughly discussed in Section 3.3.2 of \cite{javanmard2018}. We have also described it in details in Appendix \ref{app:calculatingM}. We now describe our proposed method which is basically built upon the debiased LASSO estimator in \eqref{dbe}.

By $\left(X_i,Y_i\right)$, we represent the $i$th patch of data which is located in client $i\in[N]$. In our method, each client $i$ attempts to recover the model parameters via \eqref{dbe} using only $\left(X_i,Y_i\right)$. Such local recoveries could be unreliable, e.g., associated with a high $\mathbf{FP}$, since the number of local samples $n/N$ might be relatively small. However, a central server node can have limited communications with the clients to gather, aggregate, and then refine these unreliable estimations into a far more accurate estimator for $T$.

\begin{algorithm}[t]
\label{main_alg}
\SetAlgoLined
\textbf{Stage One:}\\
\quad\quad 
\textbf{for each client $i \in [N]$\ in parallel do:}
\\
\quad\quad\quad 
$\- \hat{\theta}^{\mathrm{Lasso}}_{i} \leftarrow \argmin_{\theta} \ell_{\lambda_i}\left(\theta;X_i, Y_i\right)$
\\
\quad\quad\quad
$\hat{\theta}^{d}_{i} \leftarrow \hat{\theta}^{\mathrm{Lasso}}_i + \frac{1}{n_i} M_i X_i^T \left(Y_i-X_i\hat{\theta}^{\mathrm{Lasso}}_i\right)$ 
\\
\quad\quad\quad
${\mathcal{F}}_i \leftarrow \left\{k: \left\vert\hat{\theta}^{d,(k)}_{i}\right\vert \ge \tau_i\right\}$
\\
\quad\quad\quad
\textbf{Send} ${\mathcal{F}}_i$ to server
\newline

\vspace*{-1mm}
\textbf{Stage Two:}\\
\quad \quad \textbf{server executes}:\\
\quad \quad \quad $\hat{\mathcal{F}} \leftarrow $ majority vote between $\{{\mathcal{F}}_1,...,{\mathcal{F}}_N\}$\\
\quad \quad \quad Compute $\hat{T}$ and send it to all clients.



 \caption{The proposed method 
 (Sample size of client $i$ is denoted by $n_i$, where $n =n_1+\ldots+n_N$).}
\end{algorithm}

Our proposed method consists of two main stages which are formally explained in Algorithm \ref{main_alg}. At the first stage, each client $i$ independently solves
\begin{equation*}
\hat{\theta}^{\mathrm{Lasso}}_i\triangleq
\argmin_{\theta\in\mathbb{R}^p}~
\norm{Y_i-X_i\theta}_2^2 + \lambda_i \norm{\theta}_1,
\end{equation*}
where $\lambda_i$ belongs to a predetermined interval which will be fixed shortly. The estimator $\hat{\theta}^{\mathrm{Lasso}}_i$ is then debiased according to \eqref{dbe}, which results in $\hat{\theta}^{d}_i$. For a specific threshold $\tau_i>0$ which is fixed in Section \ref{performanceanalysis}, the following index set from client $i$ is computed and then transmitted to the server:
$$
\mathcal{F}_i\triangleq
\left\{k:~\left\vert
\hat{\theta}^{d,(k)}_{i}
\right\vert
\ge\tau_i\right\},
$$
where $\hat{\theta}^{d,(k)}_{i}$ represents the $k$th  dimension of $\hat{\theta}^{d}_i$. Therefore, $\mathcal{F}_i$ is the set of features whose magnitudes are above a given threshold $\tau_i$.

In the second stage, the server node performs a majority vote to select those feature indices (or equivalently, dimensions in $[p]$) which have appeared in at least half of the feature sets $\left\{\mathcal{F}_i\vert i\in[N]\right\}$, and then outputs them as the final estimator $\hat{T}$. The server can then share these final dimensions with all the clients in the network.


In the proceeding section, we theoretically show how Algorithm \ref{main_alg} solves this issue with a negligible communication cost. 
In a proper parameter configuration and by carefully choosing $\lambda_i$s, one can keep false positives less than $O\left(\log p\right)$ while having a non-zero true positive with high probability. This way, the average communication cost (number of transmissions per link) for our method is $O\left(\log p\right)$. This corresponds to both the cost per client for accumulating initial causal features at the server side, and also the cost of sending the final learned features back to the clients for possible post training. Note that we have already assumed, and also usually have $s_0=O\left(\log p\right)$ in real-world problems. In contrast to our method, the communication cost of many rival methods, e.g. FederatedAveraging algorithm, is $O\left(pt\right)$ where $t\ge 1$ denotes the number required iterations. This shows a significant reduction in communication cost in our work. 

\section{Performance Analysis} \label{performanceanalysis}
This section is devoted to a thorough theoretical analysis of our method. We begin by showing that for each client $i=1,\ldots,N$, there exists a positive threshold $\tau_i$, such that following events happen with high probability:
\begin{enumerate}
\item $\left\vert\mathcal{F}_i\right\vert\leq O\left(\log p\right) \ll p$,
\item $\mathcal{F}_i$ has positive intersection with $T$.
\end{enumerate}
This way, feature sets $\mathcal{F}_i$ become easy to transmit, and also their intersection, as we show later, gives us the true causal features with a high probability and very low error rate.

In order to prove claims 1) and 2) above, we estimate the expected values of the \emph{number} of truly selected features (true positive or \textbf{TP}), and falsely selected features (\textbf{FP}) in each client. Then, we show \textbf{TP} and \textbf{FP} concentrate around their respective expectations. We use the terms true positive rate (TPR) and false positive rate (FPR) to denote the \emph{probabilities} of selecting a causal and non-causal feature, respectively. Mathematically speaking,
$$
\mathrm{FPR}\triangleq \frac{\mathbb{E}[\mathbf{FP}]}{p-s_0}
\quad,\quad
\mathrm{TPR}\triangleq \frac{\mathbb{E}[\mathbf{TP}]}{s_0}.
$$

For a fixed client $i$, we establish an upper-bound for $\mathbb{E}[\textbf{FP}]$, and a lower-bound for $\mathbb{E}[\textbf{TP}]$, where we show they both decrease exponentially w.r.t. $\tau_i$. However, there exists a positive interval for $\tau_i$ in each client $i$, such that with high probability, TPR is still close to $1$ while FPR has significantly dropped toward zero. By choosing $\tau_i$ from such interval, both claims 1) and 2) hold. Ultimately, we show how a simple majority vote can recapture the true support of $\theta^*$ at the server-side.
\subsection{Specific Notations}
By $\norm{A}_\infty$ or $\infty$-norm of matrix $A$, we mean $\max_i\norm{A_{i}}_1$ where $A_i$ denotes the $i$th row of $A$, and $\abs{A}_\infty$ denotes the maximum absolute value of the entries of $A$, i.e., $\abs{A}_\infty = \max_{i,j} \abs{A_{i,j}}$. The $\ell_1$-norm of $A$ indicated by $\norm{A}_1$ equals to $\max_i \norm{A^T_i}_1$ or the maximum $\ell_1$-norm of the columns of $A$. We write $f(x) \lesssim g(x)$ to say for all $x$ in the domain of $f$, $f(x)$ is approximately bounded by $g(x)$, i.e. $f(x) \le Cg(x)$ for some positive constant $C$. 
\subsection{Analysis of Performance per Client} \label{perclient}

According to Algorithm \ref{main_alg}, the initial set of causal features in each client should be computed via a debiased LASSO estimator. In this regard, we take advantage of Theorem 3.13 in \cite{javanmard2018} (Theorem \ref{thm:debLASSO} below) to bound $\mathbb{E}[\textbf{TP}]$ and $\mathbb{E}[\textbf{FP}]$, and then we show how $\textbf{TP}$ and $\textbf{FP}$ concentrate around their respective means.

\begin{theo}[Debiased LASSO of \cite{javanmard2018}]
\label{thm:debLASSO}
For the sparse linear model of \eqref{linearmodel}, assume $X$ to have i.i.d. rows sampled from $\mathcal{N}\left(0,\Sigma\right)$ and let $s_0$ denote the support size of $\theta^*$. Let $\hat{\theta}^{\mathrm{Lasso}}$ denote the estimator defined by \eqref{lasso} with $\lambda = k \sigma \sqrt{(\log p)/n}$, where $k$ belongs to a fixed and known range of values. Also, define $\hat{\theta}^d$ as in \eqref{dbe}. Then, there exist constants $c$ and $C$ such that, for $n \ge c s_0 \log p$, the following holds for the residuals of debiased LASSO:
\begin{align*}
    \label{clientdist}
	\sqrt{n}\left(\hat{\theta}^d - \theta^*\right) = Z+\tilde{R},
\end{align*}
where $Z\vert X$ follows a Gaussian distribution with zero mean and covariance matrix $\sigma^2M\hat{\Sigma}M^T$. $\hat{\Sigma}$ denotes sample covariance matrix, and $\tilde{R}$ is a residual vector where we have $\norm{\tilde{R}}_{\infty} \le C \sigma \sqrt{\frac{s_0}{n}} \log p$ with probability at least $$1-2pe^{-c_* n/s_0} + pe^{-cn}+8p^{-1},$$ for some constant $c_* >0$.
\end{theo}
The explicit bounds in Theorem \ref{thm:debLASSO} correspond to the ideal situation where all the $n$ data points are already gathered in a single server node. Roughly speaking, Theorem \ref{thm:debLASSO} states that under a bounded sparsity condition, estimator $\hat{\theta}^d$ with high probability follows an approximate joint Gaussian distribution as
\begin{equation}
    \label{dis1}
    \hat{\theta}^d|X \sim \mathcal{N}\left(\theta^*+R, \frac{\sigma^2}{n}M\hat{\Sigma}M^T\right),
\end{equation}
where $R\triangleq\tilde{R}/\sqrt{n}$ and thus $\norm{R}_\infty \le C \sigma \frac{\log p \sqrt{s_0}}{n}$, for some constant $C$. 

In this work, we use Theorem \ref{thm:debLASSO} in each client, separately. Also from now on, all expectation operators are implicitly assumed to be conditional, i.e. $\mathbb{E}\left(\cdot\right)=\mathbb{E}\left(\cdot\vert X,\Psi\right)$, where $X$ is the random design matrix and $\Psi$ denotes the event of having $\norm{R}_\infty \leq O\left(\sigma \frac{\log p \sqrt{s_0}}{n/N}\right)$. According to parameter setting of Theorem \ref{thm:main} and the results of Theorem \ref{thm:debLASSO}, we already know
$$
\mathbb{P}\left(\Psi\right)\ge
1-2pe^{-c_* n/(Ns_0)} + pe^{-cn/N}+8p^{-1}
= 1-\tilde{O}(p^{-1}).
$$

Before bounding $\mathbb{E}[\textbf{TP}]$ and $\mathbb{E}[\textbf{FP}]$ and analyze their concentration bounds, let us obtain a high probability upper bound for $\abs{M\hat{\Sigma}M^T}_\infty$ via the following two lemmas:
\begin{lemm}
\label{norminfnorm1}
For two arbitrary matrices $A$ and $B$ which are conformable for multiplication, we have:
\begin{equation*}
    \abs{AB}_\infty \le \abs{A}_\infty \norm{B}_1
\end{equation*}
\end{lemm}
\begin{lemm}
\label{covbound}
Consider the linear model in \eqref{linearmodel}, and let $\hat{\theta}^d$ be as in \eqref{dbe}. According to \eqref{dis1}, we know that $\hat{\theta}^d |X \sim \mathcal{N}\left(\theta^* + R, \frac{\sigma^2}{n} M \hat{\Sigma}M^T\right)$ with high probability, then we have:
\begin{equation*}
\max_{i\in[p]}
    \left\vert
    {\left(M\hat{\Sigma}M^T\right)_{i,i}}
    \right\vert
    \lesssim 1 + \sqrt{\frac{\log p}{n}}
\end{equation*}
\end{lemm}
The proofs of Lemmas \ref{norminfnorm1} and \ref{covbound} are given in Appendix \ref{appendix}. 
Recall $\hat{\theta}^d_i$ as the debiased LASSO estimator of client $i\in[N]$. Then, by combining the result of Lemma \ref{covbound} with \eqref{dis1}, we have
\begin{equation}
    \label{dist2}
    \hat{\theta}^d_i|X \sim \mathcal{N}\left(\theta^* + R, \tilde{\Sigma}\right),
\end{equation}
where we have
\begin{align*}
&\abs{R}_\infty \le R_{\max}\triangleq O\left(\sigma \frac{\sqrt{s_0} \log p}{n/N}\right)\quad\quad\mathrm{and}\\
&\abs{\tilde{\Sigma}}_\infty \lesssim \sigma^2_{\max}\triangleq \frac{\sigma^2}{n/N}\left(1+\sqrt{\frac{\log p}{n/N}}\right),
\end{align*}  
with probability at least $1-O\left(p^{-1}\right)$.


Next, we use the results of Theorem \ref{thm:debLASSO} and Lemmas \ref{covbound} and \ref{norminfnorm1} for analyzing $\mathbb{E}[\textbf{TP}]$ and $\mathbb{E}[\textbf{FP}]$ in each client $i\in\left[N\right]$. The following theorem gives theoretical guarantees that as long as the number of samples per client $n/N$ and signal-to-noise ratio $\beta/\sigma$ are sufficiently large w.r.t. dimension $p$, then false positive and true positive rates can be kept near $0$ and $1$, respectively.

\begin{theo}
Assume the linear sparse model of \eqref{linearmodel}, where data samples are equally distributed among $N$ distinct clients. Also, assume the \textsc{Bounded Netowrk Size} and \textsc{Non-Negligible SNR} conditions of Theorem \ref{thm:main} are satisfied. For any client $i$, let us use the data portion $\left(X_i,Y_i\right)$ to obtain a debiased LASSO estimator $\hat{\theta}^{d}_i$ which is then used to estimate the causal features according to Algorithm \ref{main_alg}: by comparing the magnitudes of $\hat{\theta}^{d,(k)}_i,~k\in\left[p\right]$ with the threshold $\tau_i>0$. 

Assume a known function $\zeta\left(\cdot\right)=\zeta\left(\cdot\vert n/N,p,s_0\right)$, where we have
$$
\zeta\left(x\right)=O\left[
\frac{\sqrt{s_0}\log p}{n/N}+
\left(\frac{\log(x^{-1})}{n/N}\right)^{1/2}
\sqrt{1+\sqrt{\frac{\log p}{n/N}}}
\right]
$$
for $x\in\mathbb{R}_{+}$.
For arbitrary $\epsilon,\delta\in(0,1)$, assume $\zeta\left(\epsilon^{-1}\right),\zeta\left(\delta^{-1}\right)<\beta/(2\sigma)$. Also assume
$$
\tau_i\in
\left[
\sigma\zeta\left(\epsilon^{-1}\right)
~,~
\beta-\sigma\zeta\left(\delta^{-1}\right)\right],
$$
where the interval is non-empty due to above assumptions. Then, FPR is upper-bounded by $\epsilon$ and TPR is lower-bounded by $1-\delta$, with probability at least $1-O(p^{-1})$.
\label{thm:perclientbound}
\end{theo}
\begin{proof}
Let us define the random indicator variables $\hat{A}_1, ..., \hat{A}_p$ as follows:
\begin{equation*}
  \ \hat{A}_k = \begin{cases}
    1& \left\vert\hat{\theta}^{d,\left(k\right)}_i\right\vert \ge \tau_i\\
    0              & \text{otherwise}
    \end{cases},\quad
    k=1,\ldots,p.
\end{equation*}
Thus, random entities \textbf{FP} and \textbf{TP} can be rewritten as
\begin{align*}
    \textbf{FP} = \displaystyle \sum_{k\not\in T} \hat{A}_k
    \quad,\quad
    \textbf{TP} = \displaystyle \sum_{k\in T} \hat{A}_k.
\end{align*}
We have the following set of equations for $\mathbb{E}\left[\textbf{FP}\right]$:
\begin{align*}
\mathbb{E}\left[\textbf{FP}\right] & = \mathbb{E} \left[\displaystyle \sum_{k\not\in T} \hat{A}_k\right]
    = \displaystyle \sum_{k\not\in T} \mathbb{E} \left[ \hat{A}_k \right]
     = \sum_{k\not\in T} \mathbb{P}\left(\left\vert\hat{\theta}^{d,(k)}_i\right\vert \ge \tau_i\right).
\end{align*}
Based on the fundamental framework which is provided form Theorem \ref{thm:debLASSO}, all the following results inside the proof hold with a probability at least $\mathbb{P}\left(\Psi\right)\ge 1-2pe^{-c_* n/s_0} + pe^{-cn}+8p^{-1}$.
From \eqref{dist2}, we have (while noting that $\forall{k \not \in T}:~ \theta^*_k = 0$)
\begin{align}
\label{preexpfalse}
\forall_{k \not\in T}:~ &\mathbb{P}\left(\left\vert\hat{\theta}^{d,(k)}_i\right\vert \ge \tau_i\right) = 
\mathbb{P}\left(\hat{\theta}^{d,(k)}_i \ge \tau_i\right) +
\mathbb{P}\left(\hat{\theta}^{d,(k)}_i \leq -\tau_i\right)
\nonumber\\
&= 1 - \frac{1}{2}\text{erf}\left(\frac{\tau_i-R_k}{
    \left(2\tilde{\Sigma}_{k,k}\right)^{1/2}
    }\right) -\frac{1}{2} \text{erf}\left(\frac{\tau_i+R_k}{
    \left(2\tilde{\Sigma}_{k,k}\right)^{1/2}
    }\right),
\end{align}
where $\abs{R}_\infty \le R_{max}$ and $\abs{\tilde{\Sigma}}_\infty \lesssim \sigma^2_{max}$. Defining $\mathrm{erfc}\left(u\right)=1-\mathrm{erf}(u)$ for $u\in\mathbb{R}_{+}$ and assuming $\tau_i\ge R_k$, the formulation in \eqref{preexpfalse} can be bounded as follows:
\begin{align}
\mathbb{P}\left(\left\vert\hat{\theta}^{d,(k)}_i\right\vert \ge \tau_i\right)
&\leq
\mathrm{erfc}\left(
\frac{\tau_i-R_k}{\left(2\tilde{\Sigma}_{k,k}\right)^{1/2}}
\right)
\leq
\mathrm{erfc}\left(
\frac{\tau_i-R_{\max}}{\sqrt{2}\sigma_{\max}}
\right)
\nonumber\\
&\leq
\exp\left(-\left[\frac{\tau_i-R_{\max}}{\sqrt{2}\sigma_{\max}}\right]^2\right)\leq\epsilon,
\end{align}
where we have used the inequality $\mathrm{erfc}(x)\leq \exp(-x^2)$ for $x>0$. Also, it worth noting that this bound is tight in the following sense:
$$
\mathrm{erfc}(x)\ge
\sqrt{\frac{2e}{\pi}}\frac{\sqrt{\eta-1}}{\eta}e^{-\eta x^2},
$$
where $\eta>1$ can be arbitrarily chosen. We take advantage of this tightness in later parts of the paper.

Therefore, it implies that $\tau_i\ge R_{\max}+\sqrt{2}\sigma_{\max}\log^{1/2}(1/\epsilon)$, which according to the definitions of $R_{\max}$ and $\sigma_{\max}$ can be simplified as follows:
\begin{align}
\tau_i&\ge C\sigma\left[
\frac{\sqrt{s_0}\log p}{n/N}+
\left(\frac{\log(\epsilon^{-1})}{n/N}\right)^{1/2}
\sqrt{1+\sqrt{\frac{\log p}{n/N}}}
\right],
\end{align}
for some constant $C$. 

Similar to FPR, one can write  $\mathbb{E}\left[\textbf{TP}\right] = \sum_{k\in T} \mathbb{E}\hat{A}_k$. For the case of true positive rate (and unlike FPR), we wish to minimize the above term w.r.t. to $R_k$ and $\tilde{\Sigma}_{k,k}$ to lower-bound $\textbf{TP}$. Also, we know $\hat{\theta}^{d,(k)}_i$ follows a Gaussian distribution with mean $\theta^*_k+R_k$, where $\left\vert\theta^*_k\right\vert\ge\beta$.
Therefore, 
$\forall{k\in T}:~ \mathbb{P}\left(\left\vert\hat{\theta}^{d,(k)}_i\right\vert \ge \tau_i\right) =$
\begin{align}
&1-\mathbb{P}\left(\left\vert\hat{\theta}^{d,(k)}_i\right\vert \leq \tau_i\right)
\nonumber\\
\ge&
1-
\boldsymbol{1}_{\theta^*_k\ge\beta}\mathbb{P}\left(\hat{\theta}^{d,(k)}_i \leq \tau_i\right)-
\boldsymbol{1}_{\theta^*_k\leq-\beta}\mathbb{P}\left(\hat{\theta}^{d,(k)}_i \ge -\tau_i\right)
\nonumber\\
=&\frac{1}{2}-\frac{1}{2}
\text{erf}
\left(\frac{\tau_i-(\theta^*_k+R_k)}{
\left(2\tilde{\Sigma}_{k,k}\right)^{1/2}
}\right)
\nonumber\\
\ge&
1-
\frac{1}{2}\mathrm{erfc}\left(
\frac{\beta-R_{\max}-\tau_i}{\sqrt{2}\sigma_{\max}}
\right)
\ge 1-\delta,
\label{trueprob}
\end{align}
which, following the same procedure as before, implies the following bound for $\tau_i$:
\begin{align}
\tau_i
&\leq
\beta-R_{\max}-\sqrt{2}\sigma_{\max}\log^{1/2}((2\delta)^{-1})
\\
&\leq
\beta-\sigma\cdot O\left(
\frac{\sqrt{s_0}\log p}{(n/N)}+
\left(\frac{\log(\delta^{-1})}{n/N}\right)^{1/2}
\sqrt{1+\sqrt{\frac{\log p}{n/N}}}
\right).
\nonumber
\end{align}
The rest of the job is straightforward and the proof is complete.
\end{proof}

In fact, we already know from Theorem \ref{thm:debLASSO} that coefficients $\hat{\theta}^{d,(k)}_i$ of client $i$ should concentrate around their true corresponding values in $\theta^*$, which are either zero when $k\notin T$, or are larger (in magnitude) than $\beta$ when $k\in T$. Theorem \ref{thm:perclientbound} states that under certain conditions, by choosing the thresholds $\tau_i$ in an appropriate interval, which is a subset of $\left[0,\beta\right]$, we can guarantee that TPR is close to one while FPR has remained small with high probability.




In order to give some intuition to the reader, we set out to graphically illustrate the behavior of both the upper-bound for FPR and the lower-bound for TPR as a function of threshold $\tau_i$, both derived in Theorem \ref{thm:perclientbound}. Here, $n = 20$ and $p = 5000$ have been fixed, and then the number of causal dimensions $s_0$ and/or variance of additive noise $\sigma^2$ have been chosen so that the constraints imposed by the theorem are satisfied. The results are depicted in Figure \ref{fig:tpvsfp}.
As one would expect, both the TPR and FPR (together with their respective lower and upper-bounds) ultimately decrease with an exponential rate as $\tau_i$ is being increased. However, the downfall for lower-bound of TPR starts at a much larger threshold compared to the upper-bound of FPR.

\begin{figure}[t]
    \centering
    \includegraphics[width=0.48\textwidth]{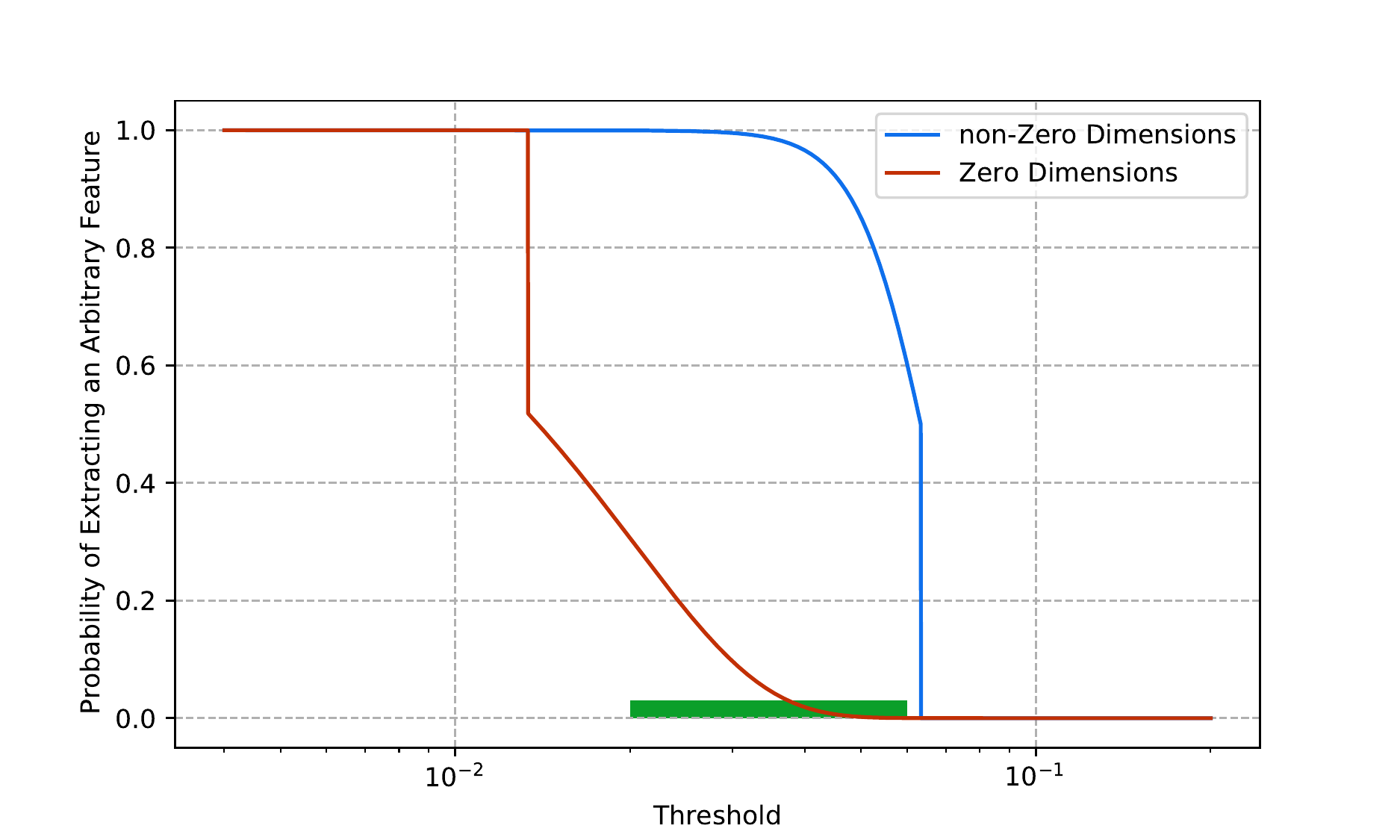}
    \caption{Behaviour of the expected FPR and TPR as a function of threshold in each client. As can be seen, there exists an interval in which TPR is high, while FPR has already declined toward zero. By choosing the threshold from this area, one can attain an intial feature set which includes the moajority of causal factors, while its size might not be too large.}
    \label{fig:tpvsfp}
\end{figure}

In any case, the existence of an appropriate interval for threshold values has been both theoretically and experimentally validated. This gives us some proper candidate sets $\mathcal{F}_i$, as defined in Algorithm \ref{main_alg}, that can be easily transmitted to the server to uncover the true causal features through a simple majority vote. Using the notation from Theorem \ref{thm:perclientbound}, first, let us clarify that the expected FPR and TPR can be easily translated to expectations of $\textbf{TP}$ and $\textbf{FP}$ as follows:
$$
\mathbb{E}[\textbf{TP}]\ge s_0\left(1-\delta\right)\quad,\quad
\mathbb{E}[\textbf{FP}]\leq \left(p-s_0\right)\epsilon.
$$
As we have already discussed, we need $\mathrm{FPR}\leq O(\log p/p)$ in each client to guarantee the vast majority of $\mathcal{F}_i$s remain relatively small in size, i.e., their cardinally remains $O(\log p)$. It should be reminded that the final FPR which would be achieved at the server node must be $O\left(\epsilon\log p/p\right)$, for arbitrary small $\epsilon$. The following corollary gives us explicit bounds for TPR and FPR as functions of primary hyper-parameters such as $n,N,p,s_0$ and signal-to-noise ratio $\beta/\sigma$.

\begin{corl}[Explicit bounds on TPR and FPR]
\label{corl:explicitFPRTPR}
Consider the setting of Theorem \ref{thm:perclientbound}. Then, a proper interval for threshold $\tau_i$ of client $i\in[N]$ exists if, and only if
\begin{equation}
\frac{\beta}{\sigma}
\ge \Theta\left(
\frac{\sqrt{s_0}\log p}{n/N}+\sqrt{\frac{\log p}{n/N}}
\right).
\label{eq:corl:asymp}    
\end{equation}
Moreover, under the above condition a FPR of at most $O\left(\log p/p\right)$ is guaranteed in each client. Let us define $\Lambda\triangleq \log p/(n/N)$. Then, FPR and TPR can be respectively written as
\begin{align}
\mathrm{FPR}&=\exp\left[-\frac{n}{N}f\left(\frac{\beta}{\sigma},s_0,\Lambda\right)\right]
\nonumber\\
\mathrm{TPR}&=1-\exp\left[-\frac{n}{N}g\left(\frac{\beta}{\sigma},s_0,\Lambda\right)\right],
\end{align}
where both $f$ and $g$ have the following asymptotic behavior:
\begin{align}
f,g = \Theta\left(s_0
\frac{\left(\frac{\beta}{\sigma\sqrt{s_0}}-O(\Lambda)\right)^2}{
{1+\sqrt{\Lambda}}
}\right)
\end{align}
\end{corl}
\begin{proof}
First, assume $\epsilon,\delta=O(1)$, then $\max\left(\epsilon^{-1},\delta^{-1}\right)=O(1)$. It can be seen that in order to have a non-empty interval for threshold $\tau_i$, the upper-bound for $\tau_i$ should exceed its lower-bound. Therefore, it is enough to have
$$
\frac{\beta}{\sigma}\ge \Theta\left(
\sqrt{s_0}\Lambda+
\left(
\frac{\log\max\left(\epsilon^{-1},\delta^{-1}\right)}{n/N}
\right)^{1/2}\sqrt{1+\sqrt{\Lambda}}.
\right)
$$
which, considering all situations: $\Lambda\rightarrow0$, $\Lambda=\Theta(1)$, $\Lambda\rightarrow\infty$, and also the fact that $\Lambda > N/n$, readily simplifies to
$$
\frac{\beta}{\sigma} \ge
\Theta\left(
\sqrt{s_0}\Lambda + \sqrt{\Lambda}
\right).
$$
Next, it is easy to see that substituting $\epsilon=O(\log p/p)$ into the bounds of Theorem \ref{thm:perclientbound}, we must have
$$
\frac{\beta}{\sigma}\ge \Theta\left(
\sqrt{s_0}\Lambda+
\left(
\frac{\log p + O(1)}{n/N}
\right)^{1/2}\sqrt{1+\sqrt{\Lambda}}.
\right)
$$
which is already satisfied as long as \eqref{eq:corl:asymp} holds.

Next, we observe that all the upper-bounds that have been used throughout the proof of Theorem \ref{thm:perclientbound} are tight. In this regard, we must have
$$
\sqrt{\frac{\log t^{-1}}{n/N}} = \Theta\left(
\frac{\beta/\sigma - O(\sqrt{s_0}\Lambda)}{\sqrt{1+\sqrt{\Lambda}}}
\right),
$$
where $t$ could be either $\epsilon$ (normalized false positive rate) or $\delta$ (one minus true positive rate).
Using simple algebra, the above relation proves the claimed formulations inside the statement of corollary. This completes the proof.
\end{proof}

According to Corollary \ref{corl:explicitFPRTPR}, it is easy to see that in each client and for asymptotically large $n$ and $p$, in order to guarantee $\mathbf{FP}=\epsilon\log p$ and $\mathbf{TP}=s_0\left(1-\delta\right)$ with probability at least $1-O\left(p^{-1}\right)$, one just needs satisfy the following condition:
$$
n/N\ge \Theta\left[
\left(\frac{\sigma}{\beta}\right)^{2}\log\frac{1}{\epsilon\delta}
\right]
\ge
\Theta\left[
\left(\frac{\sigma}{\beta}\right)^{2}\log\frac{1}{\min(\epsilon,\delta)}
\right]
.
$$
Obviously, if we manage to transmit all the data to a single node and then compute the debiased LASSO estimator, then the above condition reduces to
$
n\ge \Theta\left[
\left(\sigma/\beta\right)^{2}\log\frac{1}{\epsilon\delta}
\right]$; however, an overall of $O\left(np\right)$ transmissions all over the network would be required.

\subsection{Concentration around expected values}

So far, we have analyzed the behavior of TPR and FPR, or equivalently $\textbf{TP}$ and $\textbf{FP}$, through bounding their expected values. However, in order to complete the analysis, one also needs to show that the above quantities properly \emph{concentrate} around their expectations. Since both $\mathbf{TP}$ and $\mathbf{FP}$ are positive random variables, bounding their expectations would naturally give us high probability tail bounds for the actual variables, as well. Using Markov inequality, we have the following upper and lower-bounds for FPR and TPR, respectively.

\begin{corl}
\label{corl:concentration}
Consider the setting described in Theorem \ref{thm:perclientbound}, and assume for some $\epsilon,\delta\in(0,1)$ we have $\zeta\left(\epsilon^{-1}\right),\zeta\left(\delta^{-1}\right)<\beta/(2\sigma)$. Also, assume the threshold $\tau_i$ for client $i\in\left[N\right]$ belongs to the proper interval which has been formulated according to the theorem. Then, with probability at least
$$
1-O\left(N^{-2}+p^{-1}\right)
$$
we have $\mathbf{FP}\leq N^2\epsilon\log p$ and $\mathbf{TP}\ge s_0\left(1-N^2\delta\right)$.
\end{corl}
Proof is straightforward and can be simply obtained through applying Markov inequality on the expectation bounds of Theorem \ref{thm:perclientbound}. In some special cases, where $n/N$ can be much larger than the minimum requirements of Theorem \ref{thm:main}, far better concentration and tail bounds can be achieved through applying existing concentration inequalities for weakly-dependent random variables. Appendix \ref{app:concentration} gives a thorough analysis of such cases for an interested reader.


\subsection{Analysis of Consensus at Server-Side} \label{consensus}
According to the section \eqref{promthd}, the final stage of our proposed algorithm, which is a consensus approach, is to estimate the causal features via majority voting amongst the feature sets that have been received from all the clients. More precisely, we choose those dimensions that appear in at least half of the feature sets. In this sub-section, we calculate high probability tail bounds for $\mathbf{FP}$ and $\mathbf{TP}$ at the server-side. Assume we have in each client $i\in[N]$:
$$
\frac{\mathbf{FP}}{p-s_0}\leq \Delta\in\left(0,1/2\right).
$$
with probability at least $1-O(p^{-1}+N^{-2})$. Then, the post-consensus $\mathbf{FP}$, denoted by $\mathbf{FP}_{\mathrm{MV}}$, can be bounded as
\begin{align}
\frac{\mathbf{FP}_{\mathrm{MV}}}{p-s_0}&\leq
\sum_{i=\left\lceil{N}/{2}\right\rceil}^{N} \binom{N}{i}\Delta^i \left(1-\Delta\right)^{N-i}
\nonumber\\
&\leq \frac{N}{2}\left(\frac{Ne}{N/2}\right)^{N/2}\left(
\frac{\Delta}{1-\Delta}
\right)^{N/2}
\nonumber\\
&\leq N\left(4e\epsilon\right)^{N/2}=e^{-N\Theta\left(\log \Delta^{-1}\right)},
\end{align}
with probability at least
$$
\ge \prod_{i=1}^{N}\left[1-O\left(N^{-2}+p^{-1}\right)\right]
\ge e^{O\left(N^{-1}+Np^{-1}\right)}.
$$
Note that probabilities are simply multiplied due to the statistical independence between clients. Also, since we have $N\leq o(p)$ due to our main assumptions, this probability lower bounds converges to $1$ in the asymptotic case. Similarly, when $\mathbf{TP}\ge s_0\left(1-\Delta\right)$ (with $\Delta\in(0,1/2)$), we would have the following bound for the post-consensus $\mathbf{TP}_{\mathrm{MV}}$:
$$
\frac{\mathbf{TP}_{\mathrm{MV}}}{s_0}\ge
1-e^{-N\Theta\left(\log\Delta^{-1}\right)},
$$
with the same probability bound.

Following the results of Corollary \ref{corl:concentration}, we can easily replace $\Delta$ with $N^2\epsilon\log p/p$ and $N^2\delta$, respectively. As a result, the $\textbf{FP}$ and $\textbf{TP}$ values after the majority vote with high probability of at least $\exp\left(N^{-1}+Np^{-1}\right)$ can be bounded as 
\begin{align}
\mathbf{FP}_{\mathrm{MV}}&\leq
\left(p-s_0\right)
\exp
\left[-nf\left(\frac{\beta}{\sigma\sqrt{s_0}},\Lambda,s_0,N,p\right)
\right]
\quad \mathrm{and}
\nonumber\\
\mathbf{TP}_{\mathrm{MV}}&\ge
s_0-
s_0\exp
\left[-n
g\left(\frac{\beta}{\sigma\sqrt{s_0}},\Lambda,s_0,N,p\right)
\right],
\end{align}
where $\Lambda\triangleq \log p/(n/N)$, and
$$
f,g=
\Theta\left(s_0
\frac{\left(\frac{\beta}{\sigma\sqrt{s_0}}-O(\Lambda)\right)^2}{
{1+\sqrt{\Lambda}}
}\right)
-O\left(\frac{\log N+\log p)}{(n/N)\log\log p}\right).
$$
It is easy to see that based on the parameter setting og Theorem \ref{thm:main}, functions $f$ and $g$ both have the asymptotic behavior of $\Theta\left(\beta^2/\sigma^2\right)$. Thus, the required sample complexity for certifying an FPR of at most $\epsilon\log p/p$ and TPR of at least $1-\delta$, with probability at least $\exp\left[O\left(N^{-1}+Np^{-1}\right)\right]$, would be
$$
n\ge \Theta\left[\left(\frac{\sigma}{\beta}\right)^2\log\frac{1}{\min(\epsilon,\delta)}\right],
$$
which completes the proof of Theorem \ref{thm:main}. With respect to communication cost, it should be reminded that since $\mathbb{E}[\mathbf{FP}]$ is $O\left(\log p\right)$ in each client, and due to the statistical independence between clients, we have
\begin{align}
\mathrm{Comm.~cost}&\leq
\sum_{i=1}^{N}\left(\mathbf{TP}_i+\mathbf{FP}_i\right)+N\left(\mathbf{TP}_{\mathrm{MV}}+\mathbf{FP}_{\mathrm{MV}}\right)
\nonumber\\
&\leq O\left(N\log p\right),
\end{align}
with high probability based on Hoeffding inequality. Without using Algorithm \ref{main_alg}, and if one tries to train each node in isolation with no communication across the network, the minimum sample complexity would naturally become $N$-times larger than that of the centralized case. Therefore, sample complexity should grow linearly with $N$. On the other hand, the centralized scenario has the same (up to a constant that vanishes in the asymptotic regime) sample complexity as our method while imposes a communication cost of $O(np)$. These facts highlight the theoretically-guaranteed superiority of Algorithm \ref{main_alg} compared to the extreme scenarios of centralized and decentralized training.

What remains to do is to compare Algorithm \ref{main_alg} with a number of most prominent rivals in practice. To the best of our knowledge, the majority of well-known techniques prior to this work do not have strong theoretical guarantees in the form of Theorem \ref{thm:main}. However, there are guarantees for methods such as ADMM, as a method of parameter estimation, which are designed to solve similar problems compared to this paper. The proved communication cost bounds for either of such methods are at least $O(Np)$, which is again significantly larger than $O(N\log p)$ of our proposed method. Moreover, it should be noted that almost all existing distributed methods for reliably feature selection such as \cite{MORANFERNANDEZ201727}, \cite{9116528}, or \cite{10.5555/3225642.3225829} require at least $O(N p)$ communication bits. Nevertheless, \cite{Gui2020ADAGESAA} reduces the communication cost to less than $O(p)$, but its shrinkage with respect to total number of dimensions $p$ is not theoretically analyzed.

We can also compare our method with rival techniques through experimenting on a number of synthetic and real-world datasets. In this regard, we have devoted the next section to experimental investigation and comparison with some existing practical methods in this line of research. 

\section{Numerical Experiment
\label{sec:exp}}
This section is devoted to a number of experimental investigations into the theoretical results of previous sections. We have examined the performance of our proposed method over a number of real-world and synthetic datasets. First of all, and for the sake of clarification, let us review a number of well-known criteria for evaluating the performance of any sparse linear regression/feature selection method. The fraction of causal selected dimensions (denoted as ``precision"), and the ``power" of a method to retrieve causal dimensions are defined as
\begin{align*}
\mathrm{precision} = \frac{\mathbf{TP}}{|\hat{T}|}
\quad\mathrm{and}\quad
\mathrm{power} &= \frac{\mathbf{TP}}{s_0},
\end{align*}
respectively. Here, ``power" is also referred to as ``recall" by many researchers. Evidently, for a good feature selection model both of the above quantities should be close to $1$. Thus, we employ ``$F$-Measure", the harmonic mean of precision and recall, which is defined as
\begin{align*}
    F\mathrm{-Measure} = \frac{2\cdot\mathrm{precision}\cdot \mathrm{recall}}{\mathrm{precision} + \mathrm{recall}},
\end{align*}
to evaluate the overall performance.

\subsection{Synthetic Data}

Let us start with a simple example. Consider the linear model of \eqref{linearmodel} where each row of the design matrix $X$ is independently sampled from $\mathcal{N}(0,I)$ with $p=100$, $s_0 = 5$, and $\sigma = 10^{-2}$. We randomly choose the support set of $\theta^*$ form $[p]$, and initialize the coefficient values by uniformly sampling from $\{[-1,-\beta] \cup [\beta, 1]\}$, where $\beta=0.1$. In addition, we set the number of samples per client as $n/N=20$, such that a client would not be able to reliably recover the causal dimensions in isolation. Figure \ref{client} supports this claim, since the obtained $F$-Measure in each single client is at most around $65$ percent.
\begin{figure}
    \centering
    \includegraphics[width=0.48\textwidth]{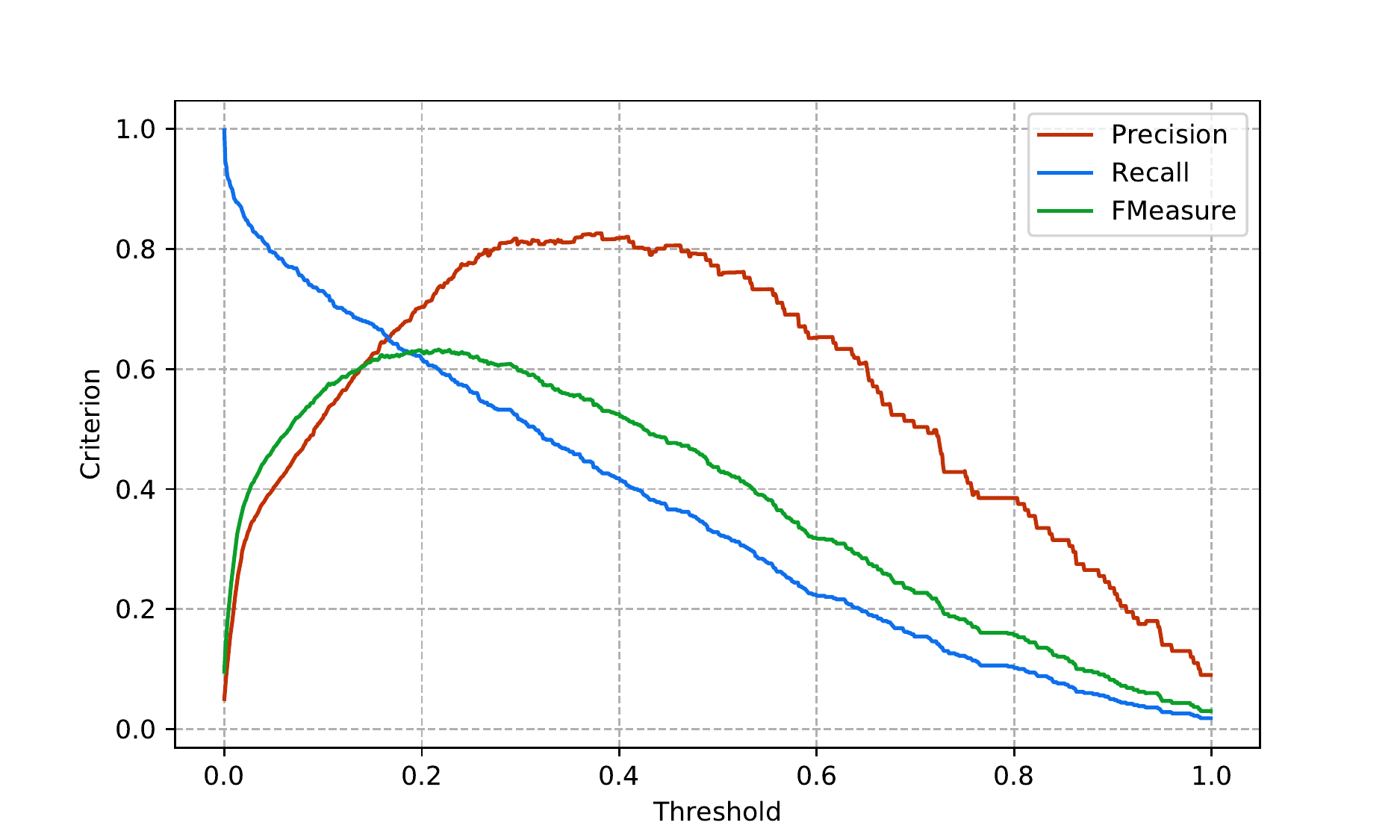}
    \caption{Based on the setting described so far, the $F$-measure in each single client does not reach an acceptable rate.}
    \label{client}
\end{figure}

However, the performance significantly improves through applying Algorithm \ref{main_alg} for extracting the support set of $\theta^*$. Figure \ref{fixed-th} and \ref{fixed-client} show the evaluation criteria of retrieved dimensions in the server as functions of the number of clients, and selection threshold, respectively. More precisely, in Figure \ref{fixed-th}, we force each client to transmit exactly $25$ dimension indices to the server which correspond to the highest absolute values of the debiased-LASSO estimator. These dimension-sets are then aggregated through a majority voting scheme as described earlier. Obviously, given that sample per client $n/N=20$ is fixed, increasing the number of clients $N$ would increase the $F$-measure which can also be seen in Figure \ref{fixed-th}. In Figure \ref{fixed-client}, we fix the number of clients $N=10$ and sample per client $n/N=20$. However, we sweep the threshold for choosing the causal dimensions in each client, i.e. $\tau$, and then compute the $F$-measure and other related criteria. We observe that for a properly-chosen range of threshold values, even $10$ clients are enough to achieve a reliable $F$-measure which is very close to $1$.
\begin{figure}
\centering
\includegraphics[width=0.48\textwidth]{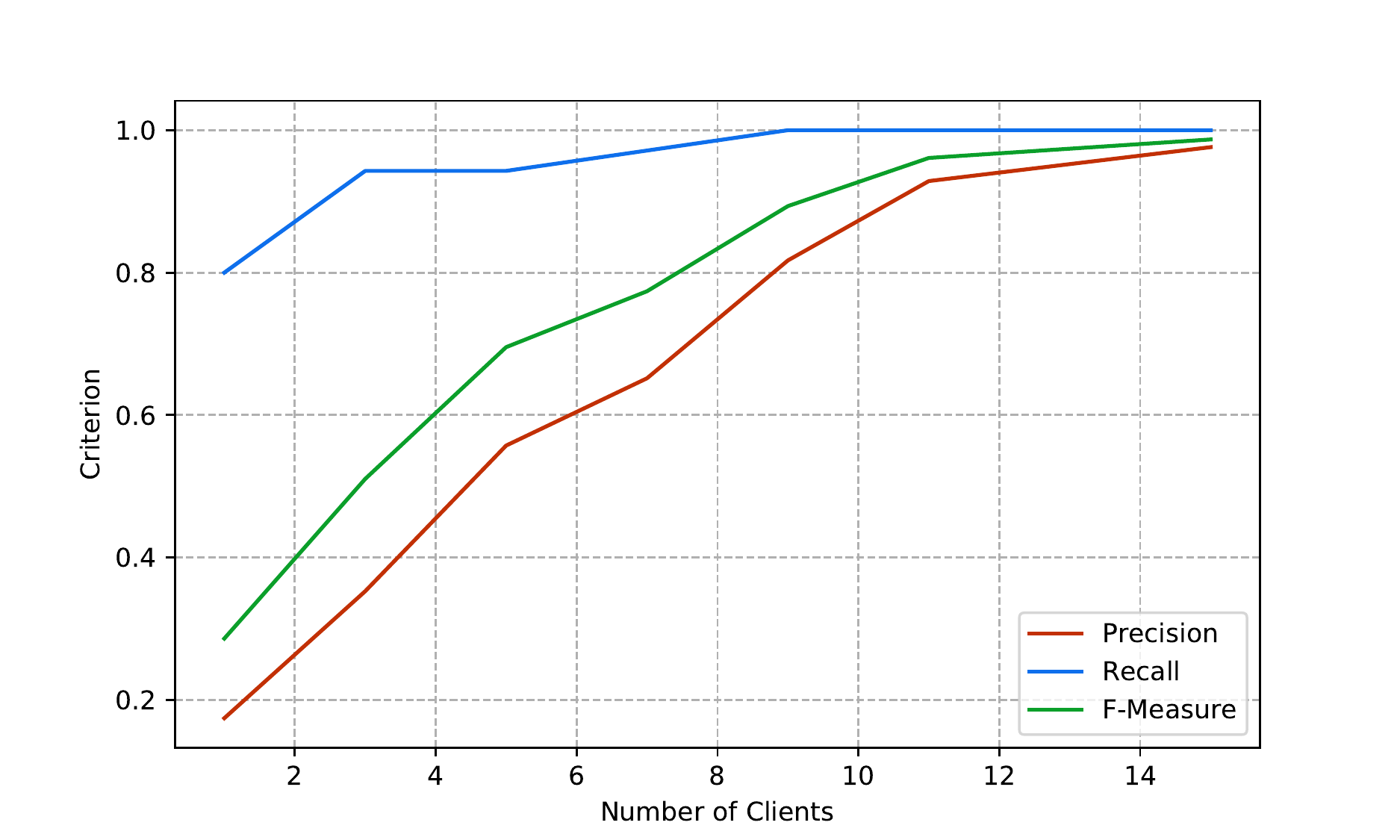}
\caption{Behaviour of accuracy, recall and $F$-measure as the number of clients is being increased. Here, we have fixed $n/N=20$, and each client sends exactly $25$ dimension indices to the server.}
\label{fixed-th}
\end{figure}

\begin{figure}
\centering
\includegraphics[width=0.48\textwidth]{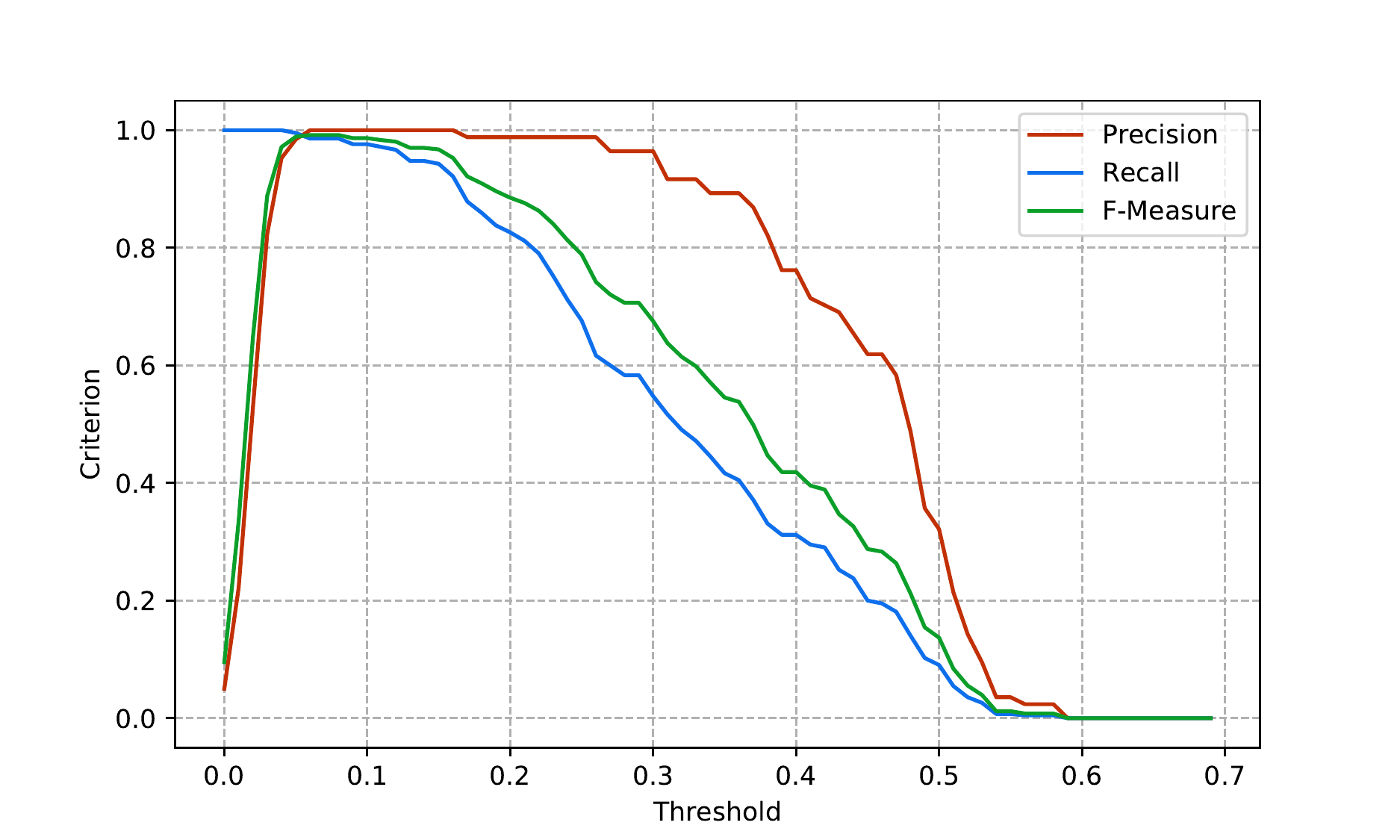}
\caption{Behaviour of accuracy, recall and $F$-measure as one increases the selection threshold $\tau$. Here, we have fixed $N=10$, and $n/N=20$. It can be seen that there exists an interval for $\tau$, where both accuracy and recall are very close to $1$.}
\label{fixed-client}
\end{figure}


\subsection{Real Data}
In this part, we set out to evaluate our method on a real-world dataset in order to detect those mutations in Human Immunodeficiency Virus Type-$1$ (HIV-1) that statistically associate with resistance to drugs\footnote{Data is available in \href{https://hivdb.stanford.edu/pages/published_analysis/genophenoPNAS2006/}{https://hivdb.stanford.edu}}. Dataset has been previously discussed and entirely investigated in \cite{Rhee17355}. The reference also contains the experimentally-approved mutation sets in the protease and Reverse Transcriptase (RT) positions of the HIV-1 sequences which correspond to 1) resistance to protease inhibitors (PI), 2) Nucleoside Reverse Transcriptase Inhibitors (NRTIs), and 3) to non-nucleoside RT inhibitors (NNRTIs). Therefore, we already know the ground truth which is necessary for evaluation of Algorithm \ref{main_alg}.

We have employed our method on one of the drugs (NFV) form the PI drug class. The observation variable, $y$, has been set to be the log-fold-increase of laboratory-tested drug resistance. This way, we can consider the linear model of \eqref{linearmodel} where the design matrix $X$ in this experiment consists of $849$ observations (rows) and $209$ mutations (features, or columns). Each row of this matrix is a sequence of $0$ and $1$, where $1$ shows the occurrence of a mutation and $0$ shows otherwise. Similar to prior works on this dataset, we only consider the mutations appearing more than $3$ times in the samples. Also, in order to deal with missing values and preparation of the data, the procedure described at the main data configuration source code\footnote{https://hivdb.stanford.edu/download/GenoPhenoDatasets/DRMcv.R} has been utilized. Our final purpose is to extract the relevant positions for drug resistance, in a distributed manner and with very low communication cost. This toy example could model a distributed and privacy-preserving collaboration format among different genetics research institutes across the world.

We have randomly distributed the samples into $N=5$ clients, and then applied Algorithm \ref{main_alg} to retrieve the most relevant mutations. 
Power (recall) and False Discovery Proportion (FDP), defined as $1-\mathrm{precision}$, are used to evaluate the results. We have compared our results with two non-distributed methods for feature extraction, which can be found in \cite{Barber_2015} and \cite{javanmard2019}, respectively. As illustrated in Figure \ref{fig:realdata}, the performance of our method is more or less comparable to rival techniques, even though in our setting data samples are distributed.
\begin{figure}
\label{fig:realdata}
\begin{center}
\includegraphics[width=0.51\textwidth]{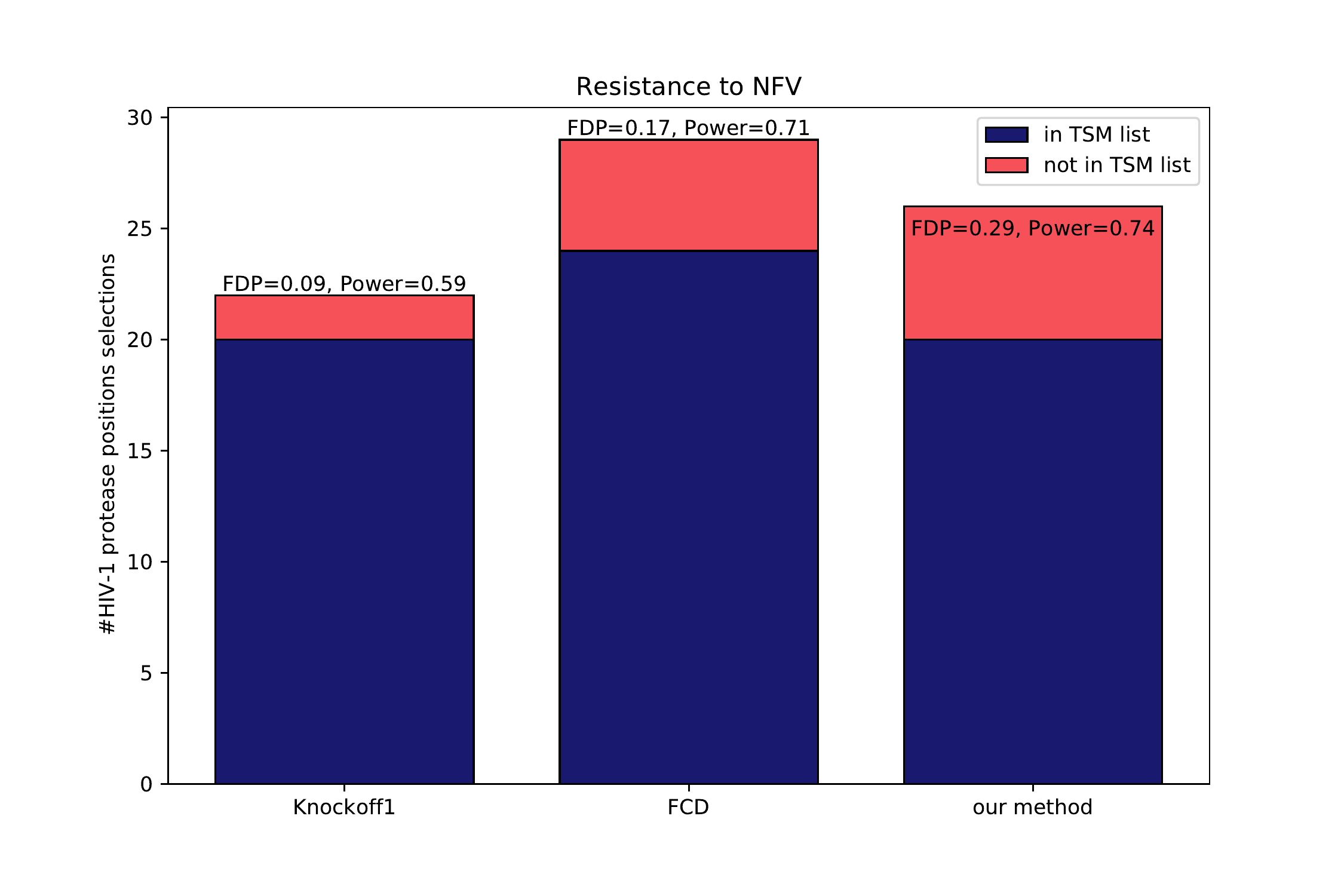}
\end{center}
\caption{Applying our method on a real dataset for HIV-1 drug resistance analysis. The blue section of each bar shows the number of truly selected mutations, while the pink section indicates otherwise. Moreover, the power and FDP of retrieving the relevant mutations are reported at the top of each bar.}
\end{figure} 




\section{Conclusion
\label{sec:conc}}
Our main contribution is to propose, and also theoretically analyze a distributed and communication-efficient scheme for feature selection in high-dimensional sparse linear models. Using our method, we prove the sample complexity of feature selection in the distributed regime has the same order as that of the centralized case whenever $\ell_1$-regularized linear regression has been utilized as the core strategy. The significance of the previous statement is the fact that such sample complexity is achievable through imposing a very low communication-cost, and without any notable extra computational burden. In particular, we show that a simple thresholding algorithm at each client, together with another simple majority-voting scheme at the server side would give us the theoretically-guaranteed low error rates of centralized scenarios, while the computation is already distributed and privacy of user data is preserved. Theoretical results in this work are based on some recent mathematical findings w.r.t. debiased LASSO estimator in \cite{javanmard2018}. Also, we have tried to validate our theoretical findings via a number of simple numerical experiments on both synthetic and real-world datasets.

\section*{Acknowledgment}

This research was supported by a grant from IPM.

\bibliographystyle{IEEEtran}
%
% ...

\bibliography{IEEEabrv,asme2e.bib}

\begin{IEEEbiography}[{\includegraphics[width=1in,height=1.25in, clip, keepaspectratio]{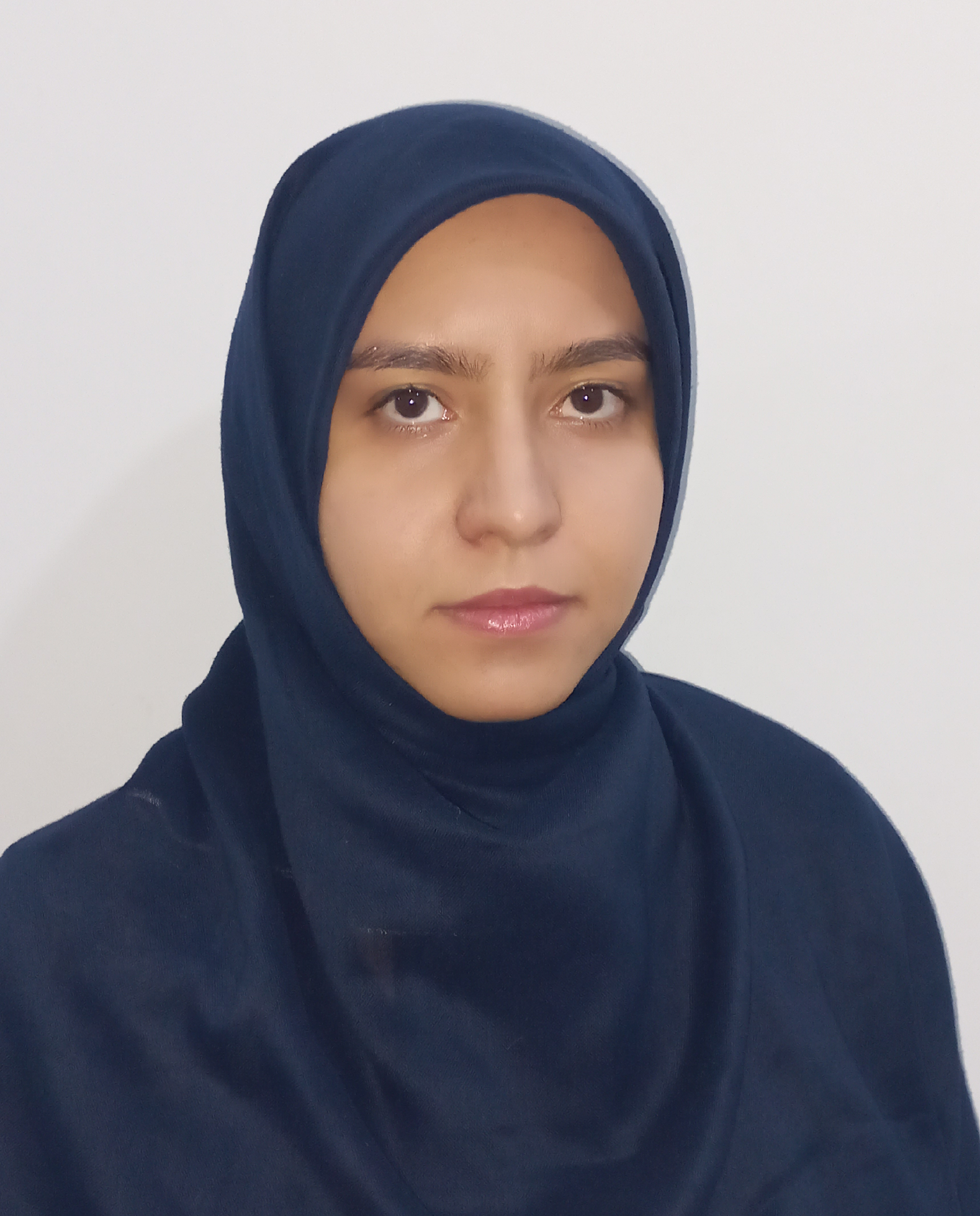}}]{Hanie Barghi}
received her B.Sc. degree in Computer Engineering from Sharif University of Technology (SUT), Tehran, Iran, in 2018. She received her M.Sc. degree in Artificial Intelligence from SUT in 2021. Her research interests include machine learning, distributed learning, and statistics.
\end{IEEEbiography}
\vfill
\begin{IEEEbiography}[{\includegraphics[width=1.2in,height=1.2in, trim=5mm 0 5mm 0,clip, keepaspectratio]{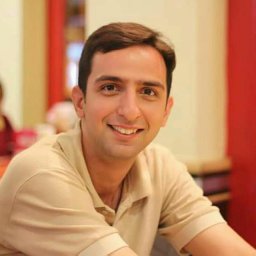}}]{Amir Najafi}
is currently a postdoctoral associate at the School of Mathematics, Institute for Research in Fundamental Sciences (IPM), Tehran, Iran. He received his B.Sc. and M.Sc. degrees in electrical engineering from Sharif University of Technology (SUT), in 2012 and 2015, respectively. He also received his Ph.D. degree in artificial intelligence from Computer Engineering Dept. of SUT in 2020. He was with the Broad Institute of MIT and Harvard, MA, USA, as a visiting research scholar in 2016, and interned at Preferred Networks Inc., Tokyo, Japan, as a researcher in learning theory group in 2018. His research interests include machine learning theory and statistics.
\end{IEEEbiography}
\vfill
\begin{IEEEbiography}[{\includegraphics[width=1.25in,height=1.25in, trim=0mm 0mm 0mm 0mm,clip, keepaspectratio]{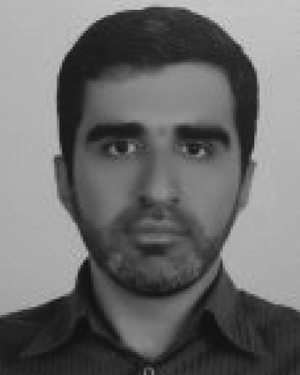}}]{Seyed Abolfazl Motahari}
received the B.Sc. degree from the Iran University of Science and Technology (IUST), Tehran, in 1999, the M.Sc. degree from Sharif University of Technology, Tehran, in 2001, and the Ph.D. degree from the University of Waterloo, Waterloo, Canada, in 2009, all in electrical engineering. He is currently an assistant professor with the Computer Engineering Department, Sharif University of Technology (SUT). From October 2009 to September 2010, he was a Post-Doctoral Fellow with the University of Waterloo, Waterloo. From September 2010 to July 2013, he was a Post-Doctoral Fellow with the Department of Electrical Engineering and Computer Sciences, University of California at Berkeley. His research interests include multiuser information theory, machine learning theory, and bioinformatics. He received several awards including Natural Science and Engineering Research Council of Canada (NSERC) Post-Doctoral Fellowship.\end{IEEEbiography}
\vfill
\appendices
\section{Calculation of the Approximate Inverse Covariance Matrix\label{app:calculatingM}}

Recall that LASSO estimator $\hat{\theta}^{\mathrm{Lasso}}$ for learning the true coefficient vector $\theta^*$ has been defined as:
\begin{equation}
    \label{lasso}
    \hat{\theta}^{\mathrm{Lasso}} = \argmin_{\theta\in\mathbb{R}^p}~\left\{
    \ell_{\lambda}\left(\theta;X,Y\right)\triangleq
    \norm{Y-X\theta}_2^2 + \lambda \norm{\theta}_1\right\},
\end{equation}
where $\lambda\ge0$ is a user-defined regularizing coefficient. In this regard, the ``debiased LASSO" estimator has been defined as an alternative which is proved to have a number of useful theoretical properties \cite{javanmard2018,zhang2011confidence,van_de_Geer_2014,6866880}:  
\begin{equation}
\hat{\theta}^d = \hat{\theta}^{\mathrm{Lasso}} + \frac{1}{n} M X^T \left(Y-X\hat{\theta}^{\mathrm{Lasso}}\right).
\end{equation}
For an unknown covariance matrix $\Sigma$, $M$ represents an estimation of the inverse covariance matrix $\Theta = \Sigma^{-1}$. In particular, we have $M = \hat{\Phi}^{-2}\hat{C}$, where the diagonal matrix $\hat{\Phi}$ is computed as
\begin{equation*}
    \hat{\Phi}^2 = \text{diag}(\hat{a}^2_1, ..., \hat{a}^2_p),~ \mathrm{with}\ \ \ \ \hat{a}_i^2 = \frac{1}{n} (\tilde{x}_i - X_{\sim i}\hat{\gamma}_i)^T\tilde{x}_i,
\end{equation*}
where $\tilde{x}_i$ is the $i$th column of design matrix $X$, and $X_{\sim i}$ is a submatrix of $X$ with $\tilde{x}_i$ being removed. Also, in \cite{javanmard2018} matrix $\hat{C}$ is defined as
\begin{equation*}
    \hat{C} = \begin{bmatrix}
    1 & -\hat{\gamma}_{1,2} & \dots &  -\hat{\gamma}_{1,p}\\
    -\hat{\gamma}_{2,1} & 1 & \dots &  -\hat{\gamma}_{2,p}\\
    \vdots & \vdots & \ddots & \vdots\\
    -\hat{\gamma}_{p,1} & -\hat{\gamma}_{p,2} & \dots &  1\\
    \end{bmatrix},
\end{equation*}
where $\hat{\gamma}_i = (\hat{\gamma}_{i,j})_{j\in[p]\backslash i} \in \mathbb{R}^{p-1}$ for $i\in [p]$ is defined as follows:
\begin{equation*}
    \hat{\gamma}_i(\tilde{\lambda}) = \displaystyle \argmin_{\gamma \in \mathbb{R}^p} \frac{1}{2n} \norm{\tilde{x}_i - X_{\sim i}\gamma}^2_2 + \tilde{\lambda} \norm{\gamma}_1.
\end{equation*}
Here, $\tilde{\lambda} = K \sqrt{\log p /n}$ with $K$ being a properly large universal constant. 


\section{Exponential Concentration for TP and FP}
\label{app:concentration}

Another look at the formulations for $\textbf{TP}$ and $\textbf{FP}$ in the proof of Theorem \ref{thm:perclientbound}, i.e.,
$$
\textbf{FP} \triangleq \displaystyle \sum_{k \not\in T} \hat{A}_k
\quad,\quad
\textbf{TP} \triangleq \displaystyle \sum_{k \in T} \hat{A}_k
$$
simply reveals that both quantities are sums of several binary random variables, and thus would naturally concentrate if the summands were independently distributed. The main challenge, however, is that $\hat{A}_k$s are not independent. In fact, $\hat{A}_k~,k\in[p]$ are the results of thresholding the components of a jointly Gaussian random vector (with high probability and up to some asymptotically small residual based on Theorem \ref{thm:debLASSO}) whose covariance matrix has non-zero, but infinitesimally small, non-diagonal entries. In other words, $\hat{A}_k$s are weakly dependent binary random variables. In order to obtain a concentration bound for the summation of such variables, one could employ Theorem 1.2 in \cite{pelekis2015hoeffdings}. 
\begin{theo}[Theorem 1.2 in \cite{pelekis2015hoeffdings}]
Suppose $H_1, ..., H_p$ to be bounded random variables with $0\le H_i \le 1$ for $i\in[p]$. Also, $H_i$s are assumed to be weakly-dependent in the following sense: there exists a constant $\gamma \in (0,1)$ such that for all $Q \subseteq [p]$ the following inequality holds:
\begin{equation}
\label{multcond}
\mathbb{E}\left(\prod_{i \in Q} H_i\right) \leq \gamma^{\abs{Q}},
\end{equation}
where $\abs{Q}$ denotes the cardinality of $Q$. Let us fix a real number $\epsilon\in(0,\frac{1}{\gamma}-1)$, and set $t = p\gamma+p \gamma \epsilon$. Then, there exists a universal constant $b \ge 1$ such that
\begin{equation*}
\mathbb{P}\left(\displaystyle \sum_{i=1}^{p} H_i \ge t\right) \le b e^{-p\mathcal{D}\left(\gamma(1+\epsilon)||\gamma\right)},
\end{equation*}
where $\mathcal{D}\left(\gamma(1+\epsilon)||\gamma\right)$ denotes the KL divergence between two Bernoulli random variables with success probabilities of $\gamma(1+\epsilon)$ and $\gamma$, respectively:
\begin{align*}
\mathcal{D}\left(\gamma(1+\epsilon)||\gamma\right)=~&
\gamma(1+\epsilon)
\log\left(\frac{\gamma(1+\epsilon)}{\gamma}\right)
\\
&+
(1-\gamma-\gamma\epsilon) \log\left(\frac{1-\gamma(1+\epsilon)}{1-\gamma}\right).
\end{align*}
\label{thm:depconcent}
\end{theo}

In this regard, we should first check if condition \eqref{multcond} holds true for random variables $\{\hat{A}_k: k \not\in T\}$ in client $i\in[N]$. It should be noted that a very similar analysis can also be preformed for random variables $\{\hat{A}_k: k \in T\}$. For all $Q \subseteq [p]-T$, we have
\begin{align}
\mathbb{E}\left(\displaystyle \prod_{k \in Q} \hat{A}_k\right) =& \mathbb{P}\left(\abs{\hat{\theta}^{d,(k)}_{i}} \ge \tau_i\big\vert~k\in Q\right)
\nonumber\\
\leq&\mathbb{P}\left(\abs{\hat{\theta}^{d,(k)}_{i}-R_k} \ge \tau_i-\abs{R_k}~\big\vert~k\in Q\right).
\label{eq:chichi}
\end{align}
For simplicity in notation, let us omit the client subscript $i$, thus we have $\tau\leftarrow\tau_i$ and $\hat{\theta}^d\leftarrow\hat{\theta}^d_i$. Also, for any vector $X\in\mathbb{R}^p$, let $X_{[Q]}$ denote a subvector of $X$ with component indices in $Q$. In a similar fashion, for matrix $W\in\mathbb{R}^{p\times p}$, let $W_{[Q]}$ denote a submatrix of $W$ with both column and row indices in $Q$.

Then, based on Theorem \ref{thm:debLASSO}, we already know that $\left(\hat{\theta}^d-R\right)_{[Q]} \sim \mathcal{N}\left(0, \tilde{\Sigma}_{[Q]}\right)$, where $\norm{R}_{\infty}$ is infinitesimally small with high probability. Therefore, we require an element-wise tail bound for a multivariate jointly Gaussian distribution with mean $R_{[Q]}$ and covariance matrix $\tilde{\Sigma}_{[Q]}$. Let $d\triangleq\abs{Q}$. In order to bound \eqref{eq:chichi}, we use the Proposition 2.1 and some subsequent derivations in \cite{multigau}, which are described through the following lemma.
\begin{lemm}[Proposition 2.1 and Eq. (3.1) of \cite{multigau}]
\label{lemma:gaussiancon}
Let $\Sigma \in \mathbb{R}^{d\times d}$ be a positive definite covariance matrix for $d\ge 2$ (case of $d=1$ is trivial), $\Theta\triangleq\Sigma^{-1}$ which exists due to the previous assumption, and $x^*$ denote the unique solution of the following constrained quadratic program:
$$
\min_{x\in\mathbb{R}^d}~\left\langle x\vert \Theta x\right\rangle
\quad\subjectto\quad
x \succcurlyeq t,
$$
where $t\in\mathbb{R}^d$ and $t\not\in (-\infty,0]^d$, $\langle\cdot\vert\cdot\rangle$ denotes the inner product, and $\succcurlyeq$ means element-wise $\ge$ operation. Then, there exists a unique index set $E \subseteq [d]$ such that
\begin{itemize}
\item
$\abs{E}=O(d)$,

\item
$x^*_{[E]} = t_{[E]} \neq \boldsymbol{0}$; And in case we have  $\abs{E}<d$, then $x^*_{[E^c]}=-\left(\Theta_{[E^c]}\right)^{-1}\Theta_{[E^c,Q]}t_{[E]} \succcurlyeq t_{[E^c]}$,

\item
we have $h_i\triangleq \left\langle e_i\big\vert \left(\Sigma_{[E]}\right)^{-1}t_{[E]}\right\rangle > 0$, for all $i\in E$,
\end{itemize}
where $e_i$ is the $i$-th unit vector in $\mathbb{R}^{|E|}$, and $E^c$ is the complement of $E$. Moreover, let us define $$\alpha\triangleq\min_{x\succcurlyeq t} \left\langle x\vert \Theta x\right\rangle > 0.$$
Now, let $H\sim\mathcal{N}\left(0,\Sigma\right)$, and set $t\triangleq\left(\tau,\ldots,\tau\right)^T\in\mathbb{R}^d$. Then, we have 
\begin{align}
\mathbb{P}\left(H\succcurlyeq t\right) 
\leq
\mathbb{P}\left(H_{[E]}\succcurlyeq t_{[E]}\right) 
\leq
\frac{e^{-\alpha/2}
}{\sqrt{
(2\pi)^{|E|}|\Sigma_{[E]}|
}} \displaystyle \prod_{i\in E} h_i^{-1}.
\label{eq:proposition2.1in23}
\end{align}
\end{lemm}
Based on Lemma \ref{lemma:gaussiancon}, and in particular \eqref{eq:proposition2.1in23}, it can be seen that when $n/N=\Theta\left(p\right)$, the expectation in \eqref{eq:chichi} drops exponentially w.r.t. $\abs{E}$, where $E$ is a specific subset of $Q$ with $\vert E\vert=O(d)$. Since $h_i$s are all strictly smaller than $1$ with high probability, it can be seen that there exists a positive constant $\chi<1$ such that for all $Q\subseteq [p]-T$, we have:
$$
\mathbb{E}\left(\displaystyle \prod_{k \in Q} \hat{A}_k\right)
\leq
\chi^{\left\vert Q\right\vert}.
$$
Exact calculation of $\chi$ goes beyond the scope of this work and our efforts have been focused on proving its existence. Combining the above result with those of Theorem \ref{thm:depconcent}, gives exponential concentration bounds (similar to Chernoff bounds for independent R.V.s) for both $\mathbf{FP}$ and $\mathbf{TP}$, whose exponents are linear in $p-s_0$ and $s_0$, respectively. This will give far better high probability guarantees compared to those given in Corollary \ref{corl:concentration}. However, the required constraints for this property, i.e., $n/N=O(p)$, does not hold in many practical purposes.


\section{Proof of Lemmas\label{appendix}}

\begin{proof}[proof of Lemma \ref{norminfnorm1}]
We can write $\abs{A\times B}_\infty$ as follows
\begin{equation*}
    \abs{A\times B}_\infty = \displaystyle \max_{i,j}{\abs{\langle A_i, B^T_j\rangle}}
\end{equation*}
Form Holder inequality, $\abs{\langle x,y\rangle} \le \norm{x}_q \norm{y}_p, \frac{1}{q}+\frac{1}{p} = 1$, we have
\begin{align*}
    \abs{A\times B}_\infty &\le \displaystyle \max_i \abs{A_i}_\infty \times \max_j \norm{B^T_j}_1
    \\
   \Rightarrow \abs{A\times B}_\infty &\le \abs{A}_\infty \times \norm{B}_1
\end{align*}
\end{proof}
\begin{proof}[Proof of Lemma \ref{covbound}]
We analyze this case in two different regimes: 1) when covariance matrix is identity matrix $I_{p\times p}$, but it has been assumed unknown for clients, and 2) when covariance matrix is known and has a bounded condition number. The condition for the first case can be further relaxed; In fact, it is enough to assume the inverse covariance matrix is sufficiently sparse which indicates the dimensions of design matrix are weakly-dependent. However, we omit further details for the sake of simplification.

For the first case, let us obtain an upper-bound for $$\abs{(M\hat{\Sigma}-I)(M^T-I)}_\infty.$$
Due to Lemma \eqref{norminfnorm1}, we have
\begin{align*}
    \abs{(M\hat{\Sigma}-I)(M^T-I)}_\infty  &\le \abs{M\hat{\Sigma}-I}_\infty \times \norm{M^T-I}_1\\
    &=^{(a)} \abs{M\hat{\Sigma}-I}_\infty \times \norm{M-I}_\infty \\
    & \lesssim^{(b)} \sqrt{\frac{\log p}{n}} \times \norm{M-I}_\infty \\
    & \lesssim^{(c)} \sqrt{\frac{\log p}{n}} \times s_\Theta \sqrt{\frac{\log p}{n}}
\end{align*}
where (a) is deduced from the definition of $\infty$-norm and norm-$1$, and $\norm{A^T}_1 = \norm{A}_\infty$. Inequality (b) is derived from (10) and Lemma (5.3) of \cite{van_de_Geer_2014}, and for the proof of inequality (c) we refer to Theorem 2.4 in \cite{van_de_Geer_2014}, where $s_{\Theta} =\max_{i\in [p]}\abs{\{j \neq i, \Theta_{i,j} \neq 0\}}$. Thus, we can simplify the above chain of inequalities as
\begin{align}
\label{thisEqLOL}
    \abs{M\hat{\Sigma}M^T - M\hat{\Sigma} + I - M^T}_\infty \lesssim s_\Theta \times \frac{\log p}{n}.
\end{align}
By adding $M\hat{\Sigma}-I$ to the l.h.s. of \ref{thisEqLOL} and using the fact that $\abs{A + B}_\infty \le \abs{A}_\infty + \abs{B}_\infty$, we have
\begin{align*}
    \abs{M\hat{\Sigma}M^T - M^T}_\infty & \lesssim s_\Theta \times \frac{\log p}{n} + \abs{M\hat{\Sigma}-I}_\infty \\
    & \lesssim s_\Theta \times \frac{\log p}{n} + \sqrt{\frac{\log p}{n}}.
\end{align*}
Using triangle inequality of $\ell_{\infty}$-norm, 
and noticing the fact that $\abs{A}_\infty = \abs{A^T}_\infty$, we have
\begin{align*}
    \abs{M\hat{\Sigma}M^T - I}_\infty &\lesssim s_\Theta \times \frac{\log p}{n} + \sqrt{\frac{\log p}{n}} + |M^T-I|_\infty \\ & = s_\Theta \frac{\log p}{n} + \sqrt{\frac{\log p}{n}} + |M-I|_\infty \\ 
    & \le s_\Theta \frac{\log p}{n} + \sqrt{\frac{\log p}{n}} + \norm{M-I}_\infty \\ 
    & \lesssim s_\Theta \times \frac{\log p}{n} + (1 + s_\Theta) \sqrt{\frac{\log p}{n}}
\end{align*}
We assumed $\Theta=\Sigma^{-1}=I$, hence $s_{\Theta} = 0$, therefore
$
     \abs{M\hat{\Sigma}M^T - I}_\infty \lesssim \sqrt{\frac{\log p}{n}}
$.
As a result, we have
\begin{equation*}
    \abs{M\hat{\Sigma}M^T}_\infty \lesssim 1 + \sqrt{\frac{\log p}{n}}.
\end{equation*}
For the second case, we prove the bound by using the Hoeffding inequality. When covariance matrix is known, one can simply set $M=\Sigma^{-1}$. This way, we have
$$
M\hat{\Sigma}M^T=
\Sigma^{-1}\hat{\Sigma}\Sigma^{-1}=
\Sigma^{-1}\left(\frac{1}{n}\sum_{i=1}^{n/N}X_iX^T_i\right)\Sigma^{-1}
$$
where, with a little abuse of notation, $X_i$ denotes the $i$th row of the design matrix $X$. As it is evident, each entry of the empirical covariance matrix $\hat{\Sigma}$ is the sum of $n$ independent random samples with unbiased means and bounded variances. More precisely, let us write
$$
\hat{\Sigma}=\Sigma + \Sigma^{1/2}\Delta\Sigma^{1/2}.
$$
Then, $\Delta_{i,j}$ for each $i,j\in[p]$ represents the sum of $n$ i.i.d. samples with zero mean and a variance of most $2$. Thus, it can be bounded as follows:
\begin{equation}
\mathbb{P}\left(
\left\vert\Delta_{i,j}\right\vert >\varepsilon
\right)
\leq
2\exp\left(-\frac{n\varepsilon^2}{2W_{i,j}^2}\right),
\end{equation}
where, for $Z,Z'\sim\mathcal{N}\left(0,1\right)$ being two independent normal R.V.s, we have
$$
W_{i,j}^2\triangleq
\mathbb{E}\left[
\left(ZZ'\right)^2
\right]=1
$$
for $i\neq j$, and
$$
W_{i,i}^2\triangleq
\mathbb{E}\left[
\left(Z^2-1\right)^2
\right]=2.
$$
for $i\in[p]$. Using the union bound, we have
$$
\mathbb{P}\left(
\left\vert
\Delta
\right\vert_{\infty}
\ge \varepsilon
\right)
\leq
2p^2 e^{-n\varepsilon^2/2},
$$
or alternatively, with a probability of at least $1-O\left(p^{-1}\right)$ we have
$$
\left\vert\Delta\right\vert_{\infty}
\lesssim
\sqrt{\frac{\log\left(2p^3\right)}{2n}}
\lesssim
\sqrt{\frac{\log p}{n}}.
$$
Finally, it can be easily seen that
$$
\left\vert
M\hat{\Sigma}M^T-\Sigma^{-1}
\right\vert_{\infty}
=
\left\vert
\Sigma^{-1/2}\Delta\Sigma^{-1/2}
\right\vert_{\infty},
$$
and since the maximum eigenvalue of $\Sigma^{-1}$ is already assumed to be bounded irrespective of $p$, the same bound (up to a different constant) holds for $\left\vert M\hat{\Sigma}M^T-\Sigma^{-1}\right\vert_{\infty}$, which completes the proof.
\end{proof}

\end{document}